\documentclass{article}

\def\TARGET{1}

\usepackage[final]{neurips_2023}

\usepackage[utf8]{inputenc} % allow utf-8 input
\usepackage[T1]{fontenc}    % use 8-bit T1 fonts
\usepackage{hyperref}       % hyperlinks
\usepackage{url}            % simple URL typesetting
\usepackage{booktabs}       % professional-quality tables
\usepackage{amsfonts}       % blackboard math symbols
\usepackage{nicefrac}       % compact symbols for 1/2, etc.
\usepackage{microtype}      % microtypography
\usepackage{xcolor}         % colors

\usepackage{amsmath}
\usepackage{amsthm}
\usepackage{blindtext}
\usepackage{algorithm}
\usepackage{algorithmic}
\usepackage{hyperref}
\usepackage{amssymb}
\usepackage{graphicx}
\usepackage{color}
\usepackage{array}
\usepackage[makeroom]{cancel}
\usepackage{changes}
\usepackage{parskip}
\usepackage{comment}
\usepackage{qcircuit}
\usepackage{braket}
\usepackage{bbm}
\usepackage{bm}
\usepackage{enumitem}   
\usepackage{multirow}
\usepackage{mathrsfs}
\usepackage{colortbl}

\usepackage[capitalize]{cleveref}

\newcommand{\gray}{\cellcolor[gray]{0.8}}

\theoremstyle{plain}  

\newtheorem{assumption}{Assumption}
\newtheorem{thm}{Theorem} %[section]
\newtheorem{lem}[thm]{Lemma}

\newtheorem{cor}[thm]{Corollary}
\newtheorem*{cor*}{Corollary}
\newtheorem*{lem*}{Lemma}
\newtheorem*{thm*}{Theorem}
\newtheorem{cnj}[thm]{Conjecture}
\theoremstyle{definition}
\newtheorem{dfn}{Definition} %[section]
\theoremstyle{remark}
\newtheorem{rmk}{Remark} %[section]

\newtheorem{fact}{Fact} %[section]

\usepackage{empheq}

\newcommand{\ba}{\begin{align}}
\newcommand{\ea}{\end{align}}

\newcommand{\argmax}{\text{argmax}}

\newcommand{\DKL}{\textit{D}_{\text{KL}}}
\newcommand{\E}{\mathbb{E}}

\crefname{thm}{Theorem}{Theorems}
\crefname{dfn}{Definition}{Definitions}
\crefname{rmk}{Remark}{Remarks}
\crefname{lem}{Lemma}{Lemmas}
\crefname{cor}{Corollary}{Corollaries}

\newcommand{\err}{\varepsilon}
\newcommand{\cA}{\mathcal{A}}
\newcommand{\cB}{\mathcal{B}}
\newcommand{\cC}{\mathcal{C}}
\newcommand{\cD}{\mathcal{D}}
\newcommand{\cE}{\mathcal{E}}
\newcommand{\cF}{\mathcal{F}}
\newcommand{\cG}{\mathcal{G}}
\newcommand{\cH}{\mathcal{H}}
\newcommand{\cI}{\mathcal{I}}

\newcommand{\cM}{\mathcal{M}}

\newcommand{\cO}{\mathcal{O}}

\newcommand{\cS}{{\mathcal{S}}}
\newcommand{\cT}{{\mathcal{T}}}

\newcommand{\cX}{\mathcal{X}}

\newcommand{\mbbI}{\mathbb{I}}

\newcommand{\mbbP}{\mathbb{P}}

\newcommand{\mbbR}{\mathbb{R}}

\newcommand{\KL}{\mathrm{KL}}

\newcommand{\ts}{\text{TS}}

\newcommand{\Reg}{\mathfrak{R}}
\newcommand{\BR}{\mathfrak{BR}}
\renewcommand{\d}[1]{\operatorname{d}\!#1}

\newcommand{\op}[1]{\operatorname{#1}}

\definecolor{darkred}{RGB}{150,0,0}
\definecolor{darkgreen}{RGB}{0,150,0}
\definecolor{darkblue}{RGB}{0,0,150}
\hypersetup{colorlinks=true, linkcolor=red, citecolor=blue, urlcolor=darkblue}

\usepackage{titletoc}

\usepackage{todonotes}

\makeatletter
\def\blfootnote{\xdef\@thefnmark{}\@footnotetext}
\makeatother

\title{Improved Bayesian Regret Bounds for Thompson Sampling in Reinforcement Learning}

\author{%
Ahmadreza Moradipari\thanks{\tiny Authors have equal contribution.} $~$ \thanks{\tiny Toyota Motor North America, InfoTech Labs, Mountain View, CA, USA, \texttt{ahmadreza.moradipari@toyota.com} } 
\And Mohammad Pedramfar$^\ast$
\thanks{\tiny Purdue University, West Lafayette, IN, USA, \texttt{mpedramf@purdue.edu}} 
\And Modjtaba Shokrian Zini$^\ast$
\thanks{\tiny \texttt{modjtaba.shokrianzini@gmail.com}} 
\And Vaneet Aggarwal
\thanks{\tiny Purdue University, West Lafayette, IN, USA, \texttt{vaneet@purdue.edu}}
}

\begin{document}

% \footnote{ $~^\amalg$ Authors have equal contribution. }

\maketitle

\begin{abstract}

In this paper, we prove the first Bayesian regret bounds for Thompson Sampling in reinforcement learning in a multitude of settings. We simplify the learning problem using a discrete set of surrogate environments, and present a refined analysis of the information ratio using posterior consistency. This leads to an upper bound of order $\widetilde{O}(H\sqrt{d_{l_1}T})$ in the time inhomogeneous reinforcement learning problem where $H$ is the episode length and $d_{l_1}$ is the Kolmogorov $l_1-$dimension of the space of environments. 
We then find concrete bounds of $d_{l_1}$ in a variety of settings, such as tabular, linear and finite mixtures, and discuss how how our results are either the first of their kind or improve the state-of-the-art.
\end{abstract}

% All headings should be lower case (except for first word and proper nouns), flush left, and bold.

\blfootnote{
\tiny This work was supported in part by the National Science Foundation under grant CCF-2149588 and Cisco, Inc. 
}

\section{Introduction}

Reinforcement Learning (RL) is a sequential decision-making problem in which an agent interacts with an unknown environment typically modeled as a Markov Decision Process (MDP) \cite{sutton2018reinforcement,bertsekas1996neuro}. 
The goal of the agent is to maximize its expected cumulative reward. This problem has a variety of applications, including robotics, game playing, resource management, and medical treatments. The key challenge in RL is to balance the so-called exploration-exploitation trade-off efficiently: exploring unseen state-action pairs to gain more knowledge about the unknown environment or exploiting the current knowledge to maximize the expected cumulative reward. 
Two efficient approaches have been developed to control this trade-off: \textit{optimism in the face of uncertainty} (OFU) and \textit{Thompson Sampling} (TS) (or Posterior Sampling (PS)). 
OFU constructs a confidence set of statistically plausible MDPs that includes the true MDP with high probability and plays an optimistic policy according to the MDP with maximum gain from this set \cite{auer2008near,tossou2019near}. 
TS samples a statistically plausible MDP from a posterior distribution and plays the optimistic policy of the sampled MDP \cite{osband2013more, osband2017posterior}. 
In this work, we focus on the latter, and by combining an information theoretical approach first introduced by \cite{russo2016information} with analysis based on posterior consistency tools, we prove state-of-the-art Bayesian regret bounds in a variety of settings.

In this paper, we start by defining the Bayesian RL problem, where transition and reward functions are Bayesian and time inhomogeneous. The Bayesian RL problem we consider is more comprehensive than in previous works, as we allow for both Bayesian transition and Bayesian rewards, and do not make any assumption on their individual prior. To simplify the learning problem, we utilize the notion of surrogate environments, which is a discretization of the environments space, and its learning task and TS regret is a proxy to that of the main problem. The construction of the surrogate environments was first introduced by \cite{hao2022regret} with an incorrect proof, which is fixed in our work by defining the surrogate environments through an optimization. 
Of main importance is the size of this new environment space. The Bayesian regret decomposes to the product of two terms, one being the cumulative mutual information of the environment and history traversed by the policy. By the well-known entropy estimation of the mutual information, this significant factor in the regret is connected to the $l_1-$dimensions ($d_{l_1}$) of the transition and reward functions space, which can be more succinctly interpreted as the $l_1-$dimension $d_{l_1}$ of the environment space. The latter is in turn estimated by the size of the space of surrogate environments.

The information ratio, representing a trade-off of exploration/exploitation, is the other significant term in the decomposition of the TS Bayesian regret.
In an improvement to \cite{hao2022regret}, our novel analysis of this ratio based on posterior consistency tools, shows that this trade-off is bounded by $H^{3/2}$, where $H$ is the episode length. This bound is general and independent of the dimension of transition/reward function space at each step, which is
is a key factor behind the advantage of our regret bound,  such as the $\sqrt{SA}$ advantage in the tabular case compared to \cite{hao2022regret}, or the lack of any restriction on the prior (e.g., Dirichlet prior) compared to \cite{osband2017posterior}.
Following a further refined approach, we finally estimate the TS Bayesian regret to be $\widetilde{O}(\lambda\sqrt{d_{l_1}T})$ for large enough $T$ in the time inhomogeneous setting. Here, a new term `value diameter' $\lambda$, which is the average difference of the optimal value functions at different states, is used in bounding the information ratio, where instead of $H^{3/2}$, we have the smaller term $\lambda H^{1/2}$.  Bounding the information ratio with $\lambda$ is a conceptual contribution of our work, which shows that the ratio is bounded by a \textit{value-dependent} term, which is in nature different from $H$ but always $\le H + 1$. Further, there exists another bound for $\lambda$; in environments where states are reachable from one another in $D$ steps, we have $\lambda \le D + 1$. In `well-connected' MDPs, one could have $D \ll H$, implying an improvement over the $H^{3/2}$ information ratio bound.

Our generic bound is abstract in terms of $d_{l_1}$, so we estimate it in more explicit terms for useful applications. \cite{hao2022regret} have bounded $d_{l_1}$ in the tabular and linear case without formalizing this notion, and while for tabular MDPs, $d_{l_1}$ was bounded by $SAH$, for linear MDPs with feature space dimension $d_f$, we investigate their claim of the bound $d_fH$. Detailed in \cref{appsec:linear_RL}, we show a counterexample to their analysis, and we manage to find a correct estimate in this setting. 
We also introduce finite mixtures MDPs and are the first to prove a TS Bayesian regret of order $\widetilde{O}(\lambda\sqrt{HmT})$, where $m$ is the number of mixtures.

Lastly, we note that our regret bound of order $\widetilde{O}(\lambda\sqrt{d_{l_1}T})$ is the first in the general nonlinear time inhomogeneous Bayesian RL setting for TS, and generalizing \cite[Conj. 1]{osband2017posterior}, we conjecture it to be optimal if $\lambda$ can be replaced by $\widetilde{O}(\sqrt{H})$.

\paragraph{Related work.}

Since the introduction of information ratio by \cite{russo2014learning,russo2016information}, a new line of research has emerged to provide tighter regret bounds for TS.  The general approach involves factoring the Bayesian regret into two components: an information ratio that captures the trade-off between optimal action selection and information gain, and a cumulative information gain term that depends on the target environment and the history of previous observations. Then, both components are bounded separately using information theoretic tools.

In the bandit setting, this analysis has been used to bound Bayesian regret for TS \cite{dong2018information, bubeck2020first}, as well as that of a new algorithm called information-directed sampling (IDS) \cite{russo2014learning, liu2018information, kirschner2021asymptotically, hao2021information, hao2022contextual}. 
This analysis has also been used in partial monitoring \cite{lattimore2019information, lattimore2021mirror} and RL with a specific Dirichlet prior and additional assumptions \cite{lu2019information, lu2020information} or when the true environment is too complicated to learn~\cite{arumugam22_decid_what_model}.
More recently, \cite{hao2022regret} studied the Bayesian regret of TS in RL without any prior assumptions for tabular MDP. 
This is the closest work to our paper and we discuss our generalization in detail in \cref{sec:bayesian_regret_bd_for_TS}. 

The Bayesian tabular MDP case has also been studied with the additional Dirichlet prior assumption in \cite{osband2017posterior}, where they achieve a regret bound matching ours. 
In an independent approach, the first non-linear Bayesian RL model was considered by \cite{fan2021model} with a regret bound of $d H^{3/2}T^{1/2}$ where $d$ is a notion of dimension of their model, but their results were limited to Gaussian process settings with linear kernels.
Finally, \cite{chakraborty2022posterior} considered general non-linear Bayesian RL models and introduced an algorithm that obtains $d H^{1 + \alpha/2}T^{1 - \alpha/2}$ where $\alpha$ is a tuning parameter and $d$ is the dimension of $\cS \times \cA \times \cS$.

It is worth noting that there is another line of work that incorporates confidence regions into TS to achieve Bayesian regret bounds that can match the best possible frequentist regret bounds by UCB in both bandit settings \cite{russo2014learning} and RL \cite{osband2017posterior,osband2014model, osband2019deep,chowdhury19_onlin_learn_kernel_markov_decis_proces}. 
However, this technique often results in a sub-optimal Bayesian regret, as the best bound known for UCB itself is not optimal.

% \if0
\begin{table}[h]
% \centering
\caption{Bayesian regret bounds for TS (i.e. PSRL)}
\resizebox{\textwidth}{!}{
\begin{tabular}{|c|c|c|c|c|}
\hline
Reference     & Tabular & Linear & General & Comments \\
\hline
\cite{osband2013more} & $\sqrt{H^3 S^2 A L}$ & - & - & - \\
\hline
\gray \cite{osband2014model} & - & - & $L^* \sqrt{d_K d_E H L}$ & 
    \begin{tabular}{c}
    Uses Eluder dimension \\
    \gray Lipschitz assumption
    \end{tabular} \\
\hline
\cite{osband2017posterior} & $\sqrt{H^3 S A L}$ & - & - & Dirichlet prior \\
\hline
\cite{lu2019information} & $\sqrt{H^3 S A L}$ & - & - & Assumptions on prior \\
\hline
\gray \cite{chowdhury19_onlin_learn_kernel_markov_decis_proces} & $L^*\sqrt{H^3 S^2 A^2 L} $ & - & $L^* \gamma \sqrt{H L}$ & 
    \begin{tabular}{c}
    Assumptions on regularity \& noise \\
    \gray Lipschitz assumption
    \end{tabular} \\
\hline
\cite{hao2022regret} & $\sqrt{H^4 S^2 A^2 L}$ & % \cellcolor{red!25} 
- & - & - \\
\hline
{\color{blue} This paper} & $\lambda \sqrt{H^2 S A L}$ & $ \lambda\sqrt{d_{l_1} H L}$ & $\lambda \sqrt{ d_{l_1} H L}$ & 
    \begin{tabular}{c}
    Assumptions 1 \& 2 \\
    Holds in the limit $L \to \infty$
    \end{tabular} \\
\hline
\end{tabular}}
{~\\ \small
As discussed in Section~4.3 of~\cite{fan2021model}, the Lipschitz term $L^*$, which is used in the grayed papers in the table, may grow exponentially in episode length.
Note that~\cite{hao2022regret} claims a regret bound for the linear setting.
However, as discussed in Appendix~\ref{appssec:incorrect_proof_linear_RL}, their proof is incorrect.
}
% \label{table:1}
\end{table}
% \fi

While our work's emphasis is on theoretical guarantees for TS, we discuss here the experiments using this algorithm. Previous works on PSRL \cite{russo2014learning, liu2018information, kirschner2021asymptotically, hao2022contextual, osband2017posterior} come with extensive experiments on TS (and/or its variants), and discussions on computational efficiency of PSRL. In particular, experiments in \cite{osband2017posterior} support the assertion that ``PSRL dramatically outperforms existing algorithms based on OFU''. In addition, PSRL with oracle access has been shown to be the most performant, esp. when compared to recent OFU based UCBVI/UCBVI-B, or even variants of PSRL such as Optimistic PSRL \cite[Fig. 1.3]{tiapkin2022optimistic}. However, an important limitation in experiments is the need for oracle access to an optimal policy, and that can not be always satisfied efficiently. Nevertheless, clever engineering can make TS work even in large scale Deep RL. Indeed, for general RL settings, the recent work \cite{sasso2023posterior} shows how to implement TS in Deep RL on the Atari benchmark and concludes that ``Posterior Sampling Deep RL (PSDRL) significantly outperforms previous state-of-the-art randomized value function approaches, its natural model-free counterparts, while being competitive with a state-of-the-art (model-based) reinforcement learning method in both sample efficiency and computational efficiency''. In summary, experiments in the literature provide enough support for the empirical performance of TS.

\section{Preliminaries}\label{sec:preliminary}
\subsection{Finite-horizon MDP}\label{ssec:finitehorizonMDP}
We follow the literature's conventions in our notation and terminology to avoid confusion when comparing results. 
The environment is a tuple $\cE= (\cS, \mu_\cS, \cA, \mu_\cA, H,  \{P_h\}_{h=1}^H, \{r_h\}_{h=1}^H)$, where $\cS$ is the topological measurable state space, $\cA$ is the topological measurable action space, $\mu_\cS$ and $\mu_\cA$ are base probability measures on $\cS$ and $\cA$ respectively, $H$ is the episode length, $P_h : \cS \times \cA \to \Delta_{\cS,\mu_\cS}$ is the transition probability kernel, and $r_h : \cS \times \cA \to \Delta_{[0,1], \operatorname{Lebesgue}}$ is the reward function, where we fix the convention $r(s,a) := \E_x[r(x|s,a)]=\int_0^1 xr(x|s,a) \d x$ as we mostly deal with its mean value. 
Notice that $\Delta_{X, \mu}$ is the set of probability distributions over $X$ that are absolutely continuous with respect to $\mu$.
We will use $\Delta_X$ when the base measure is clear from the context.
We assume $\cS$, $\cA$ are known and deterministic while the transition probability kernel and reward are unknown and random.
Throughout the paper, the implicit dependence of $P_h$ and $r_h$ on $\cE$ should be clear from the context.

Let $\Theta_h^P$ be the topological function space of $P_h$ and $\Theta^P=\Theta_1^P\times \cdots\times \Theta_H^P$ be the full function space. 
The space $\Theta_h^P$ is assumed to be separable and equipped with prior probability measure $\rho_h^P$ yielding the product prior probability measure $\rho^P=\rho_1^P\otimes\cdots \otimes \rho_H^P$ for $\Theta^P$.
The exact same definition with similar notations $\Theta_h^R, \rho_h^R, \rho^R, \Theta^R$ applies for the reward function. 
Notice the explicit assumption of time inhomogeneity in these definitions, with all `layers' $h$ being independent. 
The two sets define the set of all environments parametrized by $\Theta = \Theta_1\times \cdots \times \Theta_H $ where $\Theta_h = \Theta_h^P \times \Theta_h^R$. 
Note that the prior is assumed to be known to the learner. This setting implies that an environment $\cE$ sampled according to the prior $\rho = \rho^P\otimes \rho^R$ is essentially determined by its transition and reward functions pair $\{(P_h,r_h)\}_{h=1}^H$. 
We simplify the notation to view $\Theta$ as the set of all environments, i.e., saying $\cE \in \Theta$ should be viewed as $\{(P_h,r_h)\}_{h=1}^H \in \Theta$.
The space of all possible real-valued functions $\{(P_h,r_h)\}_{h=1}^H$ has a natural vector space structure.
Therefore it is meaningful to discuss the notion of the convex combination of environments.
We assume that $\Theta$ is a convex subspace of the space of all possible environments.
This assumption is not restrictive, since we may replace any environment space with its convex hull.
Note that we do not assume that the support of the prior is convex.

\begin{rmk}
The case of joint prior may be of interest, but to our knowledge all prior works also take $\rho^P,\rho^R$ to be independent.
\end{rmk}

\paragraph{Agent, policy and history.} 
\label{sssec:policy_trajectory}
An agent starts at an initial state $s_1^\ell$, which is fixed for all episodes $\ell$. It observes a state $s_h^\ell$ at layer $h$ episode $\ell$, takes action $a_h^\ell$, and receives reward $r_h^\ell$. The environment changes to the next random state $s_{h+1}^\ell$ with probability $P_h(s_{h+1}^\ell | s_h^\ell,a_h^\ell)$. The agent stops acting at $s_{H+1}$ and the environment is reset to its initial state.

We define $\cH_{\ell, h}$ as the history $(s_1^\ell, a_1^\ell, r_1^\ell, \ldots, s_h^\ell, a_h^\ell, r_h^\ell )$. Denote by $\cD_\ell = (\cH_{1, H}, \ldots, \cH_{\ell-1, H})$ the history up to episode $\ell$, where $\cD_1:=\emptyset$. 
Finally, let $
    \Omega_{h} =\prod_{i=1}^{h}(\cS\times \cA\times [0,1])\,
$ be the set of all possible histories up to layer $h$. 

A policy $\pi$ is represented by stochastic maps  $(\pi_1, \ldots, \pi_{H})$ where each $\pi_h: \Omega_{h-1}\times \cS \to \Delta_{\cA, \mu_\cA}$.
Let $\Pi_S$ denote the entire stationary policy class, stationary meaning a dependence only on the current state and layer and let $\Pi \subseteq \Pi_S$.

\paragraph{Value and state occupancy functions.} 
Define the value function $V^{\cE}_{h, \pi}$ as the value of the policy $\pi$ interacting with $\cE$ at layer $h$:
\begin{align}
 V^{\cE}_{h, \pi}(s) := \E_{\pi}^{\cE}\left[\sum_{h' = h}^H r_{h'}(s_{h'}, a_{h'})  \bigg | s_h = s\right]\,,
\end{align}
where $\mathbb E_{\pi}^{\cE}$ denotes the expectation over the trajectory under policy, transition, and reward functions $\pi,P_h,r_h$.
The value function at step $H+1$ is set to null, $V_{H+1,\pi}^\cE(\cdot):=0$.
We assume there is a measurable function $\pi^*_\cE : \Theta \to \Pi$ such that $V^\cE_{h,\pi^*_{\cE}}(s) = \max_{\pi\in\Pi}V_{h,\pi}^\cE(s), \ \forall s\in\cS, h\in[H]$.
The optimal policy $\pi^*$ is a function of $\cE$, making it a random variable in the Bayesian setting. 
Lastly, let the \emph{state-action occupancy probability measure} be $\mathbb P^{\cE}_\pi(s_h=s, a_h=a)$, also known as the state occupancy measure under policy $\pi$ and environment $\cE$.
It follows from the definitions that this measure is absolutely continuous with respect to $\mu_{\cS\times \cA}:=\mu_\cS \times \mu_\cA$.
Let $d_{h,\pi}^{\cE}(s,a)$ denote the Radon–Nikodym derivative so that we have  $d_{h,\pi}^{\cE}(s,a) \d \mu_{\cS \times \cA} = \d \mbbP^{\cE}_ \pi(s_h=s, a_h=a)$. We will assume throughout the paper that this density $d_{h,\pi}^{\cE}(s,a)$ is measurable and upper bounded for all $\pi,\cE,s,a,h$. The upper bound is a reasonable assumption, and it happens trivially in the tabular case ($d_{h,\pi}^{\cE}(s,a) \le SA$). 
This also happens, e.g., when one assumes that the maps $(\cE, s, a, s', h) \mapsto P^\cE_h(s' | s, a)$ and $(\pi, s, a, h) \mapsto \pi_h(a | s)$ are continuous and $\Theta$, $\cS$, $\cA$ and the set of all optimal policies (as a subset of $\Pi$) are compact.

\subsection{Bayesian regret} We formulate the expected regret over $L$ episodes and $T=LH$ total steps in an environment $\cE$ as
\begin{align}
    \Reg_L(\cE, \pi) = \mathbb E\left[\sum_{\ell=1}^L\left(V_{1,\pi^*_\cE}^\cE(s_1^\ell)-V_{1,\pi^\ell}^\cE(s_1^\ell)\right)\right]\,,
\end{align}
where the expectation is over the randomness of $\pi=\{\pi^{\ell}\}_\ell$. The Bayesian regret is $\BR_L(\pi)=\mathbb E[ \Reg_L(\cE, \pi)]$.
For Thompson Sampling (TS), the algorithm selects the optimal policy of a given sample $\cE_\ell$  picked from the posterior $ \cE_\ell\sim\mathbb P(\cE\in\cdot|\cD_\ell)$:
\begin{align}
\pi_{\ts}^\ell=\argmax_{\pi\in\Pi}V_{1,\pi}^{\cE_\ell}(s_1^\ell)\,.
\end{align}
Importantly, the law of TS aligns with the posterior, i.e., $\mbbP(\cE|\cD_\ell)=\mbbP(\pi^\ell_\ts = \pi^*_\cE|\cD_\ell)$.
\begin{rmk}
    Note that $\mbbP(\pi^\ell_\ts = \pi^*_\cE|\cD_\ell)$ is a probability for a specific measure on the space of optimal policies. To ensure that $\int_{\Pi^*} \mathbb{P}( \pi^*|\mathcal{D}_\ell) \text{d} \rho_{\Pi^*}= 1$, we need an appropriate measure $\rho_{\Pi^*}$ on $\Pi^*$. Given the law of TS, the natural choice for this measure is the push-forward of the prior measure $\rho$ under the map $star : \Theta \to \Pi^*$, where $star(\mathcal{E}) = \pi^*_{\mathcal{E}}$.
\end{rmk}

\subsection{Notations} 
For Bayesian RL, conditional expressions involving a given history $\cD_\ell$ are widely used. We adopt the notation in \cite{hao2022regret} to refer to such conditionals; let $\mathbb P_{\ell}(\cdot) := \mathbb P(\cdot|\cD_{\ell})$, $\mathbb E_{\ell}[\cdot] := \mathbb E[\cdot|\cD_\ell]$. We can rewrite the Bayesian regret as
\begin{align}\label{eq:bayesian_regret_dfn}
    \BR_L(\pi) = \sum_{\ell=1}^L \mathbb E\left[\mathbb E_\ell\left[V_{1,\pi^*_\cE}^{\cE}(s_1^\ell)-V_{1,\pi}^\cE(s_1^\ell)\right]\right]
\end{align}
and define the conditional mutual information 
$\mathbb I_\ell(X;Y) := D_{\KL}(\mathbb P((X, Y)\in \cdot|\cD_\ell)||\mathbb P(X\in\cdot|\cD_\ell)\otimes\mathbb P(Y\in\cdot|\cD_\ell))$. For a random variable $\chi$ and random policy $\pi$, the following will be involved in the information ratio:
\begin{align}\label{eq:dfn_conditional_mut_inf}
    \mathbb I_{\ell}^{\pi}(\chi; \cH_{\ell, h}) := \mathbb I_{\ell}(\chi; \cH_{\ell, h}|\pi) = \E_\pi [D_{\KL}(\mathbb P_{\ell}((\chi, \cH_{\ell, h})\in \cdot|\pi)||\mathbb P_{\ell}(\chi\in\cdot|\pi)\otimes\mathbb P_{\ell}(\cH_{\ell, h}\in\cdot|\pi))]\,,
\end{align}
Note that $\mathbb E[\mathbb I_{\ell}(X;Y)] = \mathbb I(X;Y|\cD_{\ell})$. To clarify, $\mathbb P_{\ell}(\cH_{\ell, h}\in\cdot|\pi)$ is the probability of $\cH_{\ell,h}$ being generated under $\pi$ within some environment. Given that the histories under consideration are generated by the TS algorithm, they are always generated in the true environment $\cE$ under an optimal policy $\pi^*_{\cE'}$. For $\pi=\pi^\ell_\ts$, this can be computed as $\mathbb P_{\ell}(\cH_{\ell, h}|\pi) = \int_\cE P(\cH_{\ell, h}|\pi,\cE) \d \mbbP_\ell(\cE)$, where $P(\cH_{\ell, h}|\pi,\cE)$ is an expression in terms of transition and reward functions of $\cE$ and $\pi$.

Finally, we define $\bar{\cE}_\ell$ as the mean MDP where $P_{h}^{\bar\cE_{\ell}}(\cdot|s,a)=\mathbb E_\ell[P_h^\cE(\cdot|s,a)]$ is the mean of posterior measure, and similarly for $r_{h}^{\bar\cE_{\ell}}(\cdot | s,a)=\mathbb E_\ell[r_h^\cE(\cdot |s,a)]$. We note that under the independence assumption across layers, the same is given for the state-occupancy density $d_{h,\pi}^{\bar \cE_\ell} = \E_\ell[d_{h,\pi}^\cE]$.

\section{Bayesian RL problems}\label{sec:bayesian_RL_problems}
\begin{dfn}\label{dfn:all_RLs}
A Bayesian RL in this paper refers to the time-inhomogeneous finite-horizon MDP with independent priors on transition and reward functions, as described in \cref{ssec:finitehorizonMDP}.
\end{dfn}
The Bayesian RL \textit{problem} is the task of finding an algorithm $\pi$ with optimal Bayesian regret as defined in \cref{eq:bayesian_regret_dfn}. Below we list the variations of this problem.
A setting considered by most related works such as \cite{osband2017posterior,fan2021model} is the following:
\begin{dfn}\label{dfn:time_hom}
The \textbf{time (reward) homogeneous} Bayesian RL refers to the Bayesian RL setting where the prior $\rho^P$ ($\rho^R$) is over the space $\Theta^P$ ($\Theta^R$) containing the single transition (reward) function $P$ ($r$) defining $\cE$, i.e., all layers have the same transition (reward) functions.
\end{dfn}
\begin{dfn}\label{dfn:RL_tabular}
The \textbf{tabular} Bayesian RL is a Bayesian RL where $\cS,\cA$ are finite sets.
\end{dfn}
\begin{dfn}[Linear MDP \cite{yang2019sample, jin2020provably}]\label{dfn:RL_linear}
Let $\phi^P:\cS\times \cA\to\mathbb R^{d_f^P}, \phi^R:\cS\times \cA\to\mathbb R^{d_f^R}$ be feature maps with bounded norm $\|\phi^P(s,a)\|_2,\|\phi^R(s,a)\|_2\leq 1$. The \textbf{linear} Bayesian RL is a Bayesian RL where for any $\cE=\{(P_h^\cE,r_h^\cE)\}_{h=1}^H \in \Theta$, there exists vector-valued maps $\psi_h^{P,\cE}(s),\psi_h^{R,\cE}(s)$ with bounded $l_2-$norm such that for any $(s,a)\in\cS\times \cA$, 
\begin{align}
  P_h^\cE(\cdot|s,a) = \langle \phi^P(s,a), \psi_h^{P,\cE}(\cdot)\rangle\,, \ \ r_h^\cE(\cdot|s,a) = \langle \phi^R(s,a), \psi_h^{R,\cE}(\cdot)\rangle
\end{align}
\end{dfn}
A restricted version of the finite mixtures called linear mixture was first considered in \cite{ayoub2020model} in the frequentist setting. Here, we consider the general setting.
\begin{dfn}\label{def:finite_mixtures_RL}
The \textbf{finite mixtures} Bayesian RL is a Bayesian RL where for any $h \in [H]$ there exists fixed conditional distributions $\{Z_{h,i}^P: \cS \times \cA \to \Delta_\cS\}_{i=1}^{m_h^P}$ and $\{Z_{h,i}^R: \cS \times \cA \to \Delta_{[0,1]}\}_{i=1}^{m_h^R}$, such that for any environment $\cE$ given by $\{(P_h^\cE,r_h^\cE)\}_{h=1}^H$, there exists parametrized probability distributions $\bm{a}_h^{P,\cE}: \cS \times \cA \to \Delta_{m_h^P},\bm{a}_h^{R,\cE} : \cS \times \cA \to \Delta_{m_h^R}$ such that 
\begin{align}
    P_h^\cE(\cdot|s,a) = \sum_{i=1}^{m_h^P} a_{h,i}^{P,\cE}(s,a)Z_{h,i}^P(\cdot|s,a),& \ \ r_h^\cE(\cdot|s,a) = \sum_{i=1}^{m_h^R} a_{h,i}^{R,\cE}(s,a)Z_{h,i}^R(\cdot|s,a)
\end{align}
\end{dfn}

\section{Surrogate learning}\label{sec:surrogate_learning}

Next, we  define the discretized surrogate learning problem, and bound the size of the surrogate environments space, a significant term in the regret. To do so, we need to first define the Kolmogorov dimension of a set of parametrized distributions, esp. working out the case of $l_1-$distance. In the definitions below, we  implicitly assume any required minimal measurability assumptions on the involved sets. 

\begin{dfn}\label{dfn:cov_number_kolmogorov}
    Given a set $\cF$ of $\cO-$parametrized distributions $P: \cO \to \Delta(\cS)$ over a set $\cS$ where both $\cO,\cS$ are measurable. Let $\cM(\cdot,\cdot) : \cF \times \cF \to \mbbR^{\ge 0}$ be a \textit{distance}, i.e.,  $\cM(P,Q)\ge 0 \xleftrightarrow{=} P=Q$. Then its right $\err-$covering number is the size $K_\cM(\err)$ of the smallest set $\cC_\cM(\err) = \{P_{1},\ldots,P_{K_\cM(\err)}\} \subset \cF$ such that 
\begin{align}\label{dfn:dimension_kolmogorov}
\forall P \in \cF, \ \exists P_{j} \in \cC_\cM(\err): \ \cM(P,P_{j}) \le \err\,.
\end{align}
\end{dfn}
The potential asymmetry of $\cM$ (e.g., KL-divergence) requires the notion of left/right covering number. The right covering number will be the default, so covering number will always refer to that.
\begin{dfn}
     Let $d_\cM(\err) = \log(K_\cM(\err))$. Define the Kolmogorov $\cM-$dimension $d_\cM$ of $\cF$ as
    \begin{align}\label{eq:dim_kolmogorov_dfn}
        d_\cM = \limsup_{\err \to 0}\frac{d_\cM(\err)}{\log(\frac{1}{\err})}.
    \end{align}
\end{dfn}
For $l_1(P,Q) := \sup_{o\in \cO} ||P(\cdot|o) - Q(\cdot|o)||_1$, applying \cref{dfn:cov_number_kolmogorov} to the sets $\Theta_h^P,\Theta_h^R$ with $\cO = \cS \times \cA$, and denote the respective covering numbers by $L_h^P(\err),L_h^R(\err)$  corresponding to covering sets $\cC_h^P(\err),\cC_h^R(\err)$. Similarly applying \cref{eq:dim_kolmogorov_dfn} and denote the corresponding $l_1-$dimensions by $d_{l_1,h}^P(\err),d_{l_1,h}^R(\err),d_{l_1,h}^P,d_{l_1,h}^R$ and $d_{l_1}^P:= \sum_h d_{l_1,h}^P,d_{l_1}^R:= \sum_h d_{l_1,h}^R$. The sums $d_{l_1,h}:=d_{l_1,h}^P+d_{l_1,h}^R, d_{l_1}:=d_{l_1}^P+d_{l_1}^R$ can be interpreted as the $l_1-$dimension of $\Theta_h$ and $\Theta$, i.e., the environment space.
\begin{rmk}\label{dfn:KL_cov_number_dim}
    We can also apply this framework to the KL-divergence, by $\cM_{\KL}(P,Q) := \sup_{o \in \cO }D_\KL(P(\cdot|o)||Q(\cdot||o))$. This was implicitly used by \cite{hao2022regret} to prove their regret bound in the tabular case. Note that Pinsker's lemma (\cref{lem:pinsker}) implies that the KL-divergence is larger than the squared total variance, and the latter is trivially larger than the $l_1$ distance. Therefore, $l_1-$dimension is smaller than $d_{\cM_{\KL}}$, allowing for tighter regret bounds.
\end{rmk}
We now revisit the definition of $\err-$value partitions and show their existence is guaranteed by finite $l_1-$covering numbers. These partitions are the origins of surrogate environments.
\begin{dfn}\label{dfn:err_value_partition_surrogate}
    Given $\err>0$, an $\err-$value partition for a Bayesian RL problem is a partition $\{\Theta_k\}_{k=1}^{K}$ over $\Theta$ such that for any $k\in[K]$ and $\cE, \cE'\in\Theta_k$, 
\begin{align}\label{eqn:cover}
    V_{1,\pi^*_\cE}^{\cE}(s_1^\ell)-V_{1, \pi^*_\cE}^{\cE'}(s_1^\ell)\leq \err\,.
\end{align}
A \textit{layered} $\err-$value partition is one where the transition functions are independent over layers after conditioning on $k$. Throughout this paper, we will only consider layered $\err-$value partition. We define $K_{\op{surr}}(\err)$ as the minimum $K$ for which there exists a layered $\err-$value partition. 
\end{dfn}
Inspired by \cref{eq:dim_kolmogorov_dfn}, we define the surrogate dimension as $d_{\op{surr}}=\limsup_{\err \to 0}\frac{K_{\op{surr}}(\err)}{\log(1/\err)}$. 
\begin{lem}\label{lem:l_1_dim_surrogate}
    Given a Bayesian RL, we have $K_{\op{surr}}(\err) \le \prod_h L_h^P(\err/(2H)^2) \times L_h^R(\err/(4H))$. This implies $d_{\op{surr}}\le d_{l_1}$.
\end{lem}
The above is proved in \cref{appsec:lemma_l_1_dim_surrogate}. It is hard to find $d_{\op{surr}}$, but one can estimate $d_{l_1}$, and according to the above, this acts as a proxy for $K_{\op{surr}}$. This is useful as the regret relates to  $K_{\op{surr}}$. But to show this, we need to construct \textit{surrogate environments} inside each partition, and show that learning those is almost equivalent to the original problem. Let $\zeta$ be a discrete random variable taking values in $\{1,\cdots, K_{\op{surr}}(\err)\}$ that indicates the partition $\cE$ lies in, such that $\zeta=k$ if and only if $\cE\in\Theta_k$. 
\begin{lem}\label{lem:surrogate_learning}
For any $\err-$value partition and any $\ell\in[L]$, there are random environments $\tilde \cE^*_\ell\in \Theta$ with their laws only depending on $\zeta,\cD_\ell$, such that
\begin{align}\label{eqn:regret_E_vs_surrogate}
    \mathbb E_{\ell}\left[V_{1, \pi^*_\cE}^\cE(s_1^\ell)-V_{1, \pi^\ell_{\ts}}^\cE(s_1^\ell)\right]-\mathbb E_{\ell}\left[V_{1,\pi^*_\cE}^{\tilde \cE_\ell^*}(s_1^\ell)-V_{1,\pi^\ell_{\ts}}^{\tilde \cE_\ell^*}(s_1^\ell)\right]\leq \varepsilon\,.
\end{align}
The expectation in both equations is over $\cE$ and $\pi_\ts^\ell \in \{\pi^*_{\cE'}\}_{\cE' \in \Theta}$, with both sampled independently $\sim \mbbP_\ell(\cdot)$, and the $K$ different values of $\tilde{\cE}^*_\ell$. The second expectation over $(\tilde{\cE}^*_\ell,\cE)$ is over pairs that are in the same partition, i.e., $\tilde{\cE}^*_\ell,\cE$ are independent only after conditioning on $\zeta$.
\end{lem}
We note that the proof in \cite[App. B.1]{hao2022regret} contains the use of a lemma that does not apply to construct the law of the environment $\tilde \cE^*_\ell$. More details is provided in \cref{appsec:lemma_surrogate_learning}, where we find $\tilde \cE^*_\ell$ by minimizing an expected value of $\pi^\ell_\ts$.

\section{Bayesian regret bounds for Thompson Sampling}\label{sec:bayesian_regret_bd_for_TS}
\subsection{General Bayesian regret bound}\label{ssec:generic_bound}
We start by introducing the notion of value diameter.
\begin{dfn}\label{dfn:value_diameter}
Given the environment $\cE$, its value diameter is defined as
\[
\lambda_\cE := 
\max_{1\le h\le H} (\sup_s V_{h,\pi^*_\cE}^\cE(s) - \inf_s V_{h,\pi^*_\cE}^\cE(s))
+ \max_{1 \le h \le H, s \in \cS, a \in \cA} (r^{\op{sup}}_h(s, a) - r^{\op{inf}}_h(s, a)),
\]
where $r^{\op{sup}}_h(s, a)$ (and $r^{\op{inf}}_h(s, a)$) is the supremum (and infimum) of the set of rewards that are attainable under the distribution $r_h(s, a)$ with non-zero probability.
As a special case, if rewards are deterministic, then we have $r^{\op{sup}}_h(s, a) = r^{\op{inf}}_h(s, a)$ for all $s, a$.
The (average) value diameter over $\Theta$ is denoted by $\lambda := \mathbb E_{\cE \sim \rho}[ \lambda_\cE^2 ]^{1/2}$.
\end{dfn}
As the value function is between $0$ and $H$, we have $\lambda_\cE \le H + 1$ implying $\lambda \le H + 1$. 
Note that value diameter is closely related to the notion of diameter commonly defined in finite RL problems.
Strictly speaking, for a time-homogeneous RL, it is straightforward to see that the value diameter is bounded from above by one plus the diameter \cite{puterman2014markov}. 

We now discuss the assumptions surrounding our results. 
The main technical assumption of this paper is the existence of consistent estimators, which as we will see in Appendix~\ref{appsec:posterior_consistency}, is closely related to the notion of posterior consistency:
\begin{assumption}\label{assumption:consistency}
There exists a strongly consistent estimator of the true environment given the history.
\end{assumption}
Roughly speaking, we assume that with unlimited observations under TS, it is possible to find the true environment.
For this assumption to fail, we need to have two environments that produce the same distribution over histories under TS and are therefore indistinguishable from the point of view of TS.
The precise description of this assumption is detailed in \cref{appsec:posterior_consistency}.

Another necessary technical assumption is that almost all optimal policies visit almost all state action pairs in their respective environment.
\begin{assumption}\label{assumption:non-zero-state-action}
For almost every environment $\cE \in \Theta$ and almost every $(s, a) \in \cS \times \cA$ and every $h \in [H]$, we have
\[
d_{h, \pi^*_{\cE}}^{\cE}(s, a) \neq 0.
\]
\end{assumption}
Recall that, for any environment $\cE \in \Theta$, the policy $\pi^*_\cE$ is the optimal policy of $\cE$ within the policy class $\Pi$.
Therefore, one example of how the above assumption holds is when $\Pi$ is the set of $\varepsilon$-greedy algorithms and transition functions of environments assign non-zero probability to every state. Under these assumptions, we discuss our main result and its corollaries.

\begin{thm}\label{thm:gen_TS_bd}
Given a Bayesian RL problem, for all $\err>0$, we have
\begin{align}
\BR_L(\pi_\ts) \leq 2\lambda\sqrt{\log(K_{\op{surr}}(\err))T} + L\err + T_0
\end{align}
where $T_0$ does not depend on $T$. This can be further upper bounded by
\begin{align}\label{eq:gen_regret_bd_d_l_1}
    \BR_L(\pi_\ts)\le \widetilde{O}(\lambda\sqrt{d_{l_1}T})\,.
\end{align}
for large enough $T$. Given a homogeneous $l_1$ dimension $d_{\op{hom}} = d_{l_1,h}, \forall h$,  this simplifies to
\begin{align}
    \BR_L(\pi_\ts)\leq \widetilde{O}(\lambda\sqrt{Hd_{\op{hom}}T})\,.
\end{align}
\end{thm}
\begin{rmk}
    For all regret bounds, we will replace $\lambda \le H + 1$ to compare our result. For the case of homogeneous dimensions, we obtain $\widetilde{O}(H^{3/2}\sqrt{d_{\op{hom}}T})$. Crucially, our main result shows a new conceptual understanding of the information ratio by bounding it by two terms of different nature: $H$ and $\lambda$, where the latter can be bounded by either the largest diameter of the environments or $H$.
\end{rmk}
\begin{rmk}
    Despite not impacting the asymptotics, the impact of $T_0$ can be large depending on the structure of the RL problem, and could be dominant even for large $T$s in practice.
\end{rmk}
\begin{rmk}
Considering time as a part of the state observation, one could apply this regret analysis to particular time-homogeneous settings. However, this mapping of time-inhomogeneous RLs to homogeneous ones is not surjective, hence the result above does not readily extend to time-homogeneous settings.
\end{rmk}
While \cite{fan2021model} were the first to consider a nonlinear Bayesian RL model, their bound is limited to the Gaussian process (with linear kernel) setting, while ours in the nonlinear time inhomogeneous setting makes no assumptions on the prior and is the first such bound. Our novel analysis allow us to upper bound the information ratio by $\lambda\sqrt{H}$ instead of, for example $H^{3/2}\sqrt{SA}$ (\cite{hao2022regret}) in the tabular case, improving the regret bound by a square root relevant to the dimension $d$ of the problem.

The detailed proof is given in \cref{appsec:generic_bound}. 
Following \cite{hao2022regret}, the regret (\ref{eq:bayesian_regret_dfn}) is rewritten using \cref{lem:surrogate_learning} to reduce the problem into its surrogate, and we use the well-known information-ratio trick by multiplying and dividing by the mutual information. We follow that with a Cauchy-Schwarz, summarized below
\begin{align}
    \BR_L(\pi_\ts) &\le \mathbb E\left[\sum_{\ell=1}^L\frac{\mathbb E_{\ell}\left[V_{1,\pi^*_\cE}^{\tilde \cE_\ell^*}(s_1^\ell)-V_{1,\pi^\ell_{\ts}}^{\tilde \cE_\ell^*}(s_1^\ell)\right]}{\sqrt{\mathbb I_\ell^{\pi^\ell_\ts}(\tilde \cE_\ell^*; \cH_{\ell, H})}}\sqrt{\mathbb I_\ell^{\pi^\ell_\ts}(\tilde \cE_\ell^*; \cH_{\ell, H})}\right] + L\err \\
    &\le \sqrt{\mathbb E\left[ \sum_{\ell=1}^L\frac{\left(\mathbb E_{\ell}\left[V_{1,\pi^*_\cE}^{\tilde \cE_\ell^*}(s_1^\ell)-V_{1,\pi^\ell_{\ts}}^{\tilde \cE_\ell^*}(s_1^\ell)\right]\right)^2}{\mathbb I_\ell^{\pi^\ell_\ts}(\tilde \cE_\ell^*; \cH_{\ell, H})}\right] \mathbb E\left[\sum_{\ell=1}^L\mathbb I_\ell^{\pi^\ell_\ts}(\tilde \cE_\ell^*; \cH_{\ell, H})\right]}+L\err
\end{align}
Note the cost $\err$ at each episode (\cref{lem:surrogate_learning}) in the first inequality, yielding the overall error $L\err$.
Then, we can bound the mutual information appearing in the regret term by $\mathbb E\left[\sum_{\ell=1}^L\mathbb I_\ell^{\pi^\ell_\ts}(\tilde \cE_\ell^*; \cH_{\ell, H})\right] = I_\ell^{\pi^\ell_\ts}(\tilde \cE_\ell^*; \cD_\ell) \le I_\ell^{\pi^\ell_\ts}(\zeta; \cD_\ell) \le H(\zeta) \le \log(K_{\op{surr}}(\err))$, where we used the mutual information chain rule, followed by data processing inequality to substitute $\tilde \cE_\ell^* \to \zeta$, and finally used the trivial bound by the entropy.
But the main novelty of our approach lies in our control of the first term 
\begin{align}
\Gamma_\ell (\pi_\ts^\ell) 
:= \frac{\left(\mathbb E_{\ell}\left[V_{1,\pi^*_\cE}^{\tilde \cE_\ell^*}(s_1^\ell) - V_{1,\pi^\ell_{\ts}}^{\tilde \cE_\ell^*}(s_1^\ell)\right]\right)^2}
{\mathbb I_\ell^{\pi^\ell_\ts}(\tilde \cE_\ell^*; \cH_{\ell, H})}
\end{align}
called the information ratio. In our analysis, we have the following bound on its expectation.
\[
\mathbb E [ \Gamma_\ell (\pi_\ts^\ell) \mid \cE_0 ]
\leq \mathbb E \left[ \sum_h \int \frac
    { \mathbb E_\ell \left[ (\lambda_\cE d_{h, \pi^*}^{\bar\cE_\ell}(s, a) )^2 \right] }
    { \mathbb E_\ell\left[ d_{h, \pi^*}^{\bar\cE_\ell}(s, a) \right] } \mu_{\cS \times \cA}
\mid \cE_0 \right],
\]
where the average is taken over all histories $\cD_\ell$ that are generated from running TS on the true environment $\cE_0$, and we have introduced the smaller term $\lambda_\cE$ instead of $H$ in \cite{hao2022regret}. While \cite{hao2022regret} essentially bound the above only in the tabular setting with $SAH^3$, we manage to generally bound the above with a more precise bound using Doob's consistency theorem.
Assumption~\ref{assumption:consistency} allows us to use Doob's consistency theorem to conclude that for almost every environment $\cE_0$, almost every infinite sequence of histories $(\cD_\ell)_{\ell = 1}^\infty$ sampled from $\cE_0$, and every integrable function $f$, the posterior mean $\mathbb E_\ell [f(\cE)] = \mathbb E [f(\cE) \mid \cD_\ell]$ converges to $f(\cE_0)$.\nocite{ghosal2017fundamentals}
In particular, we conclude that $\mathbb E [ \Gamma_\ell (\pi_\ts^\ell) \mid \cE_0 ]$ tends to $\lambda_{\cE_0}^2 H$ in the limit, allowing us to claim that for large enough $\ell$, the expected information ratio $\mathbb E [ \Gamma_\ell (\pi_\ts^\ell) ]$ is uniformly bounded by $2\E[\lambda_\cE^2]H =  2\lambda^2 H$. As there are $L$ many such ratios, the two bounds together yield $2\sqrt{\lambda^2 HL} \cdot \sqrt{\log(K_{\op{surr}}(\err))} + L\err$. 
This bound is true for large enough $\ell$, giving the additional additive term $T_0$ in the theorem.
Since this term is additive, applying \cref{lem:l_1_dim_surrogate} to bound $\log(K_{\op{surr}}(\err))$, we have successfully shown the asymptotic behavior of the regret, independent of the prior, is of order $\widetilde{O}(H\sqrt{d_{l_1}T})$.

\subsection{Applications}\label{sec:applications}
In each application below, the challenge is to bound $d_{l_1}$ using the specifics of the model, and except for the case of tabular Bayesian RL, such analysis has not been carried out rigorously. We  formalize the corollaries and show they are state-of-the-art compared to the literature.

\paragraph{Tabular RL. }The result below follows from \cref{thm:gen_TS_bd}; the main contribution comes from our new information ratio bound, followed by the estimate $\widetilde{O}((\frac{1}{\err})^{SAH})$ of $K_{\op{surr}}(\err)$ (\cite{hao2022regret}).
\begin{cor}\label{cor:tabular_bayesian_RL}
    Given a tabular Bayesian RL problem, for large enough $T$,
    \begin{align}
    \BR_L(\pi_\ts)\leq \widetilde{O}(\lambda\sqrt{HSAT})\,,
    \end{align}
    where the polylogarithmic terms are explicitly in terms of $H,S,A,L$.
\end{cor}
We observe that our result matches \cite{osband2017posterior} when their result in the time homogeneous setting (\cref{dfn:time_hom}) is extended to time inhomogeneous. However, in that paper, the authors assume a Dirichlet based prior which we do not.
\paragraph{Linear RL. }  A previous state-of-the-art $\widetilde{O}(d_fH^{3/2}\sqrt{T})$ was claimed by \cite{hao2022regret} to hold for linear Bayesian RLs with deterministic reward. We note:
\begin{itemize}
    \item As in the previous cases, their proof in bounding their information ratio includes a factor of $d_f$, which ours avoids.
    \item We show that the proof bounding $K_{\op{surr}}(\err)$ in \cite[App. B.4]{hao2022regret} is incorrect, starting with a wrong application of Cauchy-Schwarz and a wrong mutual information in their definition of information ratio. We provide counterexamples for the estimates found therein to substantiate our claim (see \cref{appssec:incorrect_proof_linear_RL}).
\end{itemize}
To state our own corollary in this case, we need to define a few notions. Let $d_{l_1}^f=d_{l_1}^{P,f}+d_{l_1}^{R,f}$ be the sum of the $l_1-$dimensions of the feature map space $\{\psi_{h}^{P,\cE}\}_{\cE \in \Theta},\{\psi_{h}^{R,\cE}\}_{\cE \in \Theta}$ where the $l_1-$distance between feature maps is defined as $l_1(\psi_{h}^{\cE},\psi_h^{\cE'}) = \int_s \|\psi_{h}^{\cE}-\psi_h^{\cE'}\|_1\mu_\cS$. Our corollary also provides a concrete bound in the case of \textit{mixture} linear Bayesian RL where the feature maps are themselves a sum of finitely many \textbf{fixed} feature maps. This means for all $\cE \in \Theta$, we have
\begin{align}
    \psi_h^{P,\cE} = \sum_{i=1}^{m_h^P} a_{h,i}^{P,\cE} \Psi_{h,i}^P(s), \ \ \psi_h^{R,\cE}  = \sum_{i=1}^{m_h^R} a_{h,i}^{R,\cE} \Psi_{h,i}^R(s)
\end{align}
where $\{\Psi_{h,i}^P(s)\}_{i=1}^{m_h^P},\{ \Psi_{h,i}^R(s)\}_{i=1}^{m_h^R}$ are finitely many fixed feature maps and $\forall \cE,h: \sum_i |a_{h,i}^{P,\cE}|^2, \sum_i |a_{h,i}^{R,\cE}|^2 \le C_a$ for some constant $C_a>0$. Let $M= M^P+M^R = \sum_h m_h^P+
\sum_h m_h^R$.
\begin{cor}\label{cor:linear_RL}
    For a linear Bayesian RL, for large enough $T$,
    \begin{align}
    \BR_L(\pi_\ts)\leq \widetilde{O}(\lambda\sqrt{d_{l_1}^fT}).
    \end{align}    
    Given a linear Bayesian RL with finitely many states and total feature space dimension $d_f=d_f^P+d_f^R$, we have $d_{l_1} \le 2d_fHS$, yielding for large enough $T$,
    \begin{align}
    \BR_L(\pi_\ts)\leq \widetilde{O}(\lambda\sqrt{Hd_fST}).
    \end{align}
    Given a mixture linear Bayesian RL, for large enough $T$,
\begin{align}\label{eq:linear_RL_mixture}
        \BR_L(\pi_\ts)\leq \widetilde{O}(\lambda\sqrt{MT})\,,
    \end{align}
\end{cor}
 The proof is given in \cref{appsec:linear_RL}. The fact that $d_{l_1}$ appears instead of $d_f$ in the general bound is not counter-intuitive, as we should expect the complexity of the feature map space $\{\psi_h^{P,\cE}(s)\}_{\cE \in \Theta,h\in [H]},\{\psi_h^{R,\cE}(s)\}_{\cE \in \Theta,h\in [H]}$ to play a role in the regret, especially as this space can be very complex, and model very different environments that can not be grouped in the same $\err-$value partition.

Therefore, opposite to the claim made by \cite{hao2022regret}, this complexity can not be captured by simply $d_f$ except maybe in degenerate cases, such as when $\cS$ is finite, which is our second statement. More generally, if each feature map $\psi_h^{P,\cE}(s),\psi_h^{R,\cE}(s)$ can be characterized with a vector of uniformly bounded norm $\bm{a}_h^{P,\cE} \in \mbbR^{m_h^P},\bm{a}_h^{R,\cE} \in \mbbR^{m_h^R}$, then we can bound the regret in terms of $m_h^P,m_h^R$'s, as is done in \cref{eq:linear_RL_mixture} (the finite state case corresponds to $m_h^P=d_f^PS, m_h^R=d_f^RS$).

\paragraph{Finite mixtures RL. } To state our finite mixtures model result, we need to set the following notations. Let $d_{l_1}^m = d_{l_1}^{m,P}+d_{l_1}^{m,R} =  \sum_h d_{l_1,h}^{m,P}+\sum_h d_{l_1,h}^{m,R}$ correspond to the total $l_1-$dimension of the space of mixtures coefficient maps $\{\bm{a}_{h}^{P,\cE}(s,a)\}_{\cE \in \Theta},\{\bm{a}_{h}^{R,\cE}(s,a)\}_{\cE \in \Theta}$ with $l_1-$ distance defined as $l_1(\bm{a}_{h}^\cE,\bm{a}_{h}^{\cE'}) = \sup_{s,a} \|\bm{a}_{h}^\cE(s,a)-\bm{a}_{h}^{\cE'}(s,a)\|_1$. Define also the restricted finite mixtures model where $\bm{a}_{h}^{P,\cE},\bm{a}_{h}^{R,\cE}$ are vectors in $\mbbR^{m_h^P},\mbbR^{m_h^R}$ independent of $(s,a)$ and let $M=M^P+M^R = \sum_h m_h^P+ \sum_h m_h^R$.
\begin{cor}\label{cor:finite_mixtures_RL}
    Given a finite mixtures Bayesian RL problem, for large enough $T$,
\begin{align}
\BR_L(\pi_\ts)\leq \widetilde{O}(\lambda\sqrt{d_{l_1}^mT})\,.
\end{align}
Assuming the restricted finite mixtures model, for large enough $T$,
\begin{align}\label{eq:finite_mixture_restricted}
    \BR_L(\pi_\ts)\leq \widetilde{O}\left(\lambda\sqrt{MT}\right)\,.
\end{align}
which, given a uniform dimension $m=m_h^P=m_h^R$, yields $\widetilde{O}(\lambda\sqrt{HmT})$.
\end{cor}
We prove the above in \cref{appsec:finite_mixtures_RL}, deriving it from our generic bound, after relating the $l_1-$dimension $d_{l_1}$ of the environment space to that of the mixtures coefficients. To the best of our knowledge, this is the first bound for finite mixtures Bayesian RL problems. We note that in a previous work (\cite{ayoub2020model}), a restricted version of finite mixtures, like in \cref{eq:finite_mixture_restricted}, was considered in the frequentist setting. 

We finish this section by proposing the following conjecture, in line with \cite[Conj. 1]{osband2017posterior}.
\begin{cnj}
For the Bayesian RL, the following is true and optimal for \textbf{all} $T$:
\begin{align}
\BR_L(\pi_\ts)\leq O\left(\inf_{\err>0} (\sqrt{H\log(K_{\op{surr}}(\err))T} + L\err)\right) \,.    
\end{align}
where the constant factor is independent of the prior. This means there exists a Bayesian RL problem such that $\BR_L(\pi_\ts) = \widetilde{\Omega}(\sqrt{Hd_{\op{surr}}T})$. All polylogarithmic terms are in terms of $H,d_{\op{surr}},T$.
\end{cnj}
Note that the above coincides with the lower bound for the (model-based) time inhomogeneous frequentist setting; see e.g., \cite{jin2018q} for the proven lower bound for the tabular case. 
This is also $\sqrt{H}$ higher (this factor being baked in $d_{\op{surr}}$) than that of the time homogeneous frequentist setting, which is expected, according to \cite[App. D]{jin2018q}. 
Note that in this conjecture, the $\lambda$ in our bound is replaced by $\sqrt{H}$, and the conjecture is not for $T$ large enough, but for all $T$. 
Supporting this conjecture requires experiments where TS can be exactly implemented assuming access to an oracle which provides the optimal policy for a query environment. 
Simulations have been performed for the similar \cite[Conj. 1]{osband2017posterior} in the time homogeneous case. Our conjecture is similar but with the additional expected factor of $\sqrt{H}$ due to time inhomogeneity, thus their simulation also supports the above.

\section{Conclusions}

In this paper, we have addressed the Bayesian Reinforcement Learning (RL) problem in the context of time inhomogeneous transition and reward functions. By considering both Bayesian transition and Bayesian rewards without prior assumptions, we have extended the scope of previous works, making our formulation more comprehensive. To simplify the learning problem, we have introduced surrogate environments, which discretize the environment space. We have established a connection between the size of this new environment space and the $l_1$-dimensions of the transition and reward functions space, providing insights into the $l_1$-dimension of the environment space denoted by $d_{l_1}$. We have employed posterior consistency tools to analyze the information ratio, which captures the trade-off between exploration and exploitation. We conjecture that (at least a weakened version of) our posterior consistency assumption should hold in general, which is left for future work. Our analysis has resulted in a refined approach to estimate the Bayesian regret in Thompson Sampling (TS), yielding a regret bound of $\widetilde{O}(\lambda\sqrt{d_{l_1}T})$ for large enough time steps $T$. The result is specialized to linear, tabular, and finite mixtures MDPs.

{\bf Limitations: }While the paper provides asymptotic generic regret bound for TS in a generalized setup which improve the state of the art results, finding lower bounds, esp. one dependent on $\lambda$, are left open. In addition, the issue of prior misspecificity is not discussed and left for future studies. 

% \newpage
% \section*{References}
\if\TARGET0
    \bibliographystyle{plain}
\else
    \bibliographystyle{apalike}
\fi
\bibliography{refs}
\newpage

\appendix

\section{Related works}

In the related works section of the main text, we mostly focused on Bayesian regret.
Here we include a brief paragraph on bounds on frequentist regret.

For the frequentist setting, various algorithms with provable regret guarantees have been proposed for model-free tabular MDPs. 
These include UCBVI \cite{azar2017minimax}, optimistic Q-learning \cite{jin2018q}, RLSVI \cite{russo2019worst, zanette2020frequentist}, and UCB-Advantage \cite{zhang2020almost}. 
These algorithms were further generalized to linear or linear mixture MDPs, such as LSVI-UCB \cite{jin2020provably}, OPPO \cite{cai2020provably}, and UCRL-VTR \cite{ayoub2020model, zhou2021nearly}. 
Slightly more related to our work, model-based frequentist bounds have also been shown for a variant of posterior sampling (PS) in the tabular setting \cite{agrawal2017optimistic}. 
For the specific variant of \textit{optimistic} PSRL, the optimal bound in the tabular setting with a Dirichlet prior was shown in \cite{tiapkin2022optimistic}. 
To our knowledge, a frequentist bound for PS is still an open problem for general RLs. 
Minimax regret bounds have also been studied for variants of TS, as in \cite{dann2021provably}. 
Most recently, \cite{agarwal2022vo} presented VO$Q$L, an algorithm that achieves the optimal bound of $\widetilde{O}(d\sqrt{HT})$ in the general 
model-free nonlinear setting, where $d$ represents the generalized Eluder dimension of the value function space. 
Note that the notion of dimension used in our regret bounds is different, and unrelated, to the Eluder dimension used in model-free estimations. 
For frequentist model-based, the optimal bound was achieved in the tabular setting by \cite{azar2017minimax}. 
As another research direction, \cite{duan2021optimal} utilized kernel-Hilbert spaces to estimate the value of infinite horizon Markov reward process (MRP) for RL problem, and \cite{abedsoltan2023toward} pave the way for scalability challenges in kernel models.

\section{Proof of Lemma \ref{lem:l_1_dim_surrogate}}\label{appsec:lemma_l_1_dim_surrogate}

To avoid conflict with the environment space notation $\Theta= \Theta_1 \times \cdots \times \Theta_H$, we adopt the notation $\Theta_k^\err$ to refer to $\err-$value partitions.

\paragraph{Proof of covering number estimate. } ~\\
Let $\{B_{h,i}^P(\err/(2H)^2)\}_{i=1}^{L_h^P(\err/(2H)^2)},\{B_{h,i}^R(\err/(4H))\}_{i=1}^{L_h^R(\err/(4H))}$ be the $\err-$balls giving an $\err-$covering for $\Theta_h^P,\Theta_h^R$. 
Then define the $\err-$value partition $\cup_{k=1}^{K}\Theta_{k}^\err = \Theta$ where $K = \prod_h L_h^P(\err/(2H)^2) \times L_h^R(\err/(4H))$ as follows.
Each $k \in [K]$ can be enumerated as a $2H-$tuple $(i_1,j_1, \ldots, i_H,j_H)$ where $i_h \in [L_h^P(\err/(2H)^2)], j_h \in [L_h^R(\err/(4H))]$. Define $\Theta_k^\err = \{\cE | P_h^{\cE} \in B_{h,i_h}^P(\err/(2H)^2) , r_h^{\cE} \in B_{h,j_h}^R(\err/(4H))\}$.
It is straightforward to check that $\cup_k \Theta_k^\err = \Theta$.
Any environment appearing redundantly can be removed from all but one of the $\Theta_k^\err$'s it lives in, so that we have a true partition of $\Theta_k^\err$.

Next, we will need to use the following lemma.

Proving that our partition is an $\err-$value partition requires us to show that for any $\cE,\cE' \in \Theta_k^\err: V_{1,\pi^*_\cE}^{\cE}(s_1^\ell)-V_{1, \pi^*_\cE}^{\cE'}(s_1^\ell) \le \err$. We have
\begin{align}
    V_{1,\pi^*_\cE}^{\cE}(s_1^\ell)-V_{1, \pi^*_\cE}^{\cE'}(s_1^\ell) 
    =& \sum_{h=1}^H\mathbb E_{\pi^*_\cE}^{ \cE'}\left[\mathbb E_{s'\sim P_h^{ \cE}(\cdot|s_h,a_h)}[V_{h+1,\pi^*_\cE}^{ \cE}(s')]-\mathbb E_{s'\sim P_h^{ \cE'}(\cdot|s_h,a_h)}[V_{h+1,\pi^*_\cE}^{ \cE}(s')]\right] \nonumber \\ 
    &\quad+ \sum_{h=1}^H \E_{\pi^*_\cE}^{\cE'}[r_h^\cE(s_h,a_h)-r_h^{\cE'}(s_h,a_h)],
\end{align} 
Rewrite the first term and bound it as follows:
\begin{align}
\begin{split}
    \sum_{h=1}^H\mathbb E_{\pi^*_{\cE}}^{ \cE'}\left[\int_\cS P_h^{ \cE}(s'|s_h^\ell,a_h^\ell)V_{h+1,\pi^*_{\cE}}^{ \cE}(s')-\int_\cS P_h^{\cE'}(s'|s_h^\ell,a_h^\ell)V_{h+1,\pi^*_{\cE}}^{ \cE}(s')\right] \\
    \le \sum_{h=1}^H\mathbb E_{\pi^*_{\cE}}^{ \cE'}\left[\left(\int_\cS \Big|P_h^{ \cE}(s'|s_h^\ell,a_h^\ell)-P_h^{\cE'}(s'|s_h^\ell,a_h^\ell)\Big| V_{h+1,\pi^*_{\cE}}^{ \cE}(s')\right)\right].
\end{split}
\end{align}
where integrals are with respect to the measure on $\cS$. Then, we can bound probability transitions terms by 
\begin{align}\label{eq:bound_surr_l_1_transitions}
\begin{split}
    \sum_{h=1}^H &\mathbb E_{\pi^*_{\cE}}^{ \cE'}\left[\left(\int_\cS \Big|P_h^{ \cE}(s'|s_h^\ell,a_h^\ell)-P_h^{\cE'}(s'|s_h^\ell,a_h^\ell)\Big| V_{h+1,\pi^*_{\cE}}^{ \cE}(s')\right)\right] \\
   &\le H \sum_{h=1}^H\mathbb E_{\pi^*_{\cE}}^{ \cE'}\left[\int_\cS \Big|P_h^{ \cE}(s'|s_h^\ell,a_h^\ell)-P_h^{\cE'}(s'|s_h^\ell,a_h^\ell)\Big|\right] \\
   &\le H \sum_{h=1}^H \sup_{s,a}\left(\int_\cS \Big|P_h^{ \cE}(s'|s,a)-P_h^{\cE'}(s'|s,a)\Big|\right) \\
   &= H \sum_{h=1}^H l_1(P_h^{ \cE},P_h^{\cE'})
   \le H(2\frac{\err}{4H^2} \cdot H)
\end{split}
\end{align}
and similarly, reward terms by
\begin{align}\label{eq:bound_surr_l_1_rewards}
\begin{split}
\sum_{h=1}^H E_{\pi^*_{\cE}}^{ \cE'}[r_h^\cE(s_h^\ell,a_h^\ell)-r_h^{\cE'}(s_h^\ell,a_h^\ell)]
&= \sum_{h=1}^H E_{\pi^*_{\cE}}^{ \cE'}\left[\int_0^1 x\left(r_h^\cE(x|s_h^\ell,a_h^\ell)-r_h^{\cE'}(x|s_h^\ell,a_h^\ell)\right)\d x\right] \\
&\le \sum_{h=1}^H E_{\pi^*_{\cE}}^{ \cE'}\left[\int_0^1 \Big|x\left(r_h^\cE(x|s_h^\ell,a_h^\ell)-r_h^{\cE'}(x|s_h^\ell,a_h^\ell)\right)\Big|\d x\right] \\
&\le \sum_{h=1}^H \sup_{s,a}\int_0^1 \Big|x\left(r_h^\cE(x|s,a)-r_h^{\cE'}(x|s,a)\right)\Big|\d x \\
&= \sum_{h=1}^H l_1(r_h^\cE,r_h^{\cE'}) 
\le H(2\frac{\err}{4H})
\end{split}
\end{align}

See also \cref{rmk:l_1_dist_reward_fact} for the reward term bound. In the first inequality, we used that $V_{h+1,\pi^*_{\cE}}^{ \cE}(s')$ is always bounded by $H$, and in both cases we used the fact that the transition and reward functions of $\cE,\cE'$ live inside the same balls, with their $l_1$ distance being at most twice the radii $\err/(2H)^2$ and $\err/(4H)$, respectively.  Adding up the above two estimates equals $\err$, as desired. This shows that our $\Theta_k^\err$ partition is an $\err-$value partition, hence $K_{\op{surr}}(\err)\le K=\prod_h L_h^P(\err/(2H)^2) \times L_h^R(\err/(4H))$.

\begin{rmk}\label{rmk:l_1_dist_reward_fact}
    Notice that the $l_1-$distance of two reward functions is over their probability distributions, and is larger than their expected norm difference, i.e., 
\begin{align}\label{eq:l_1_dist_reward_fact}
l_1(r,r') 
&= \sup_{s,a \in \cS \times \cA} ||r(\cdot|s,a) - r'(\cdot|s,a)||_1 
= \sup_{s,a \in \cS \times \cA} \int_{0}^1 |r(x|s,a) - r'(x|s,a)| \nonumber \\
&\ge \sup_{s,a \in \cS \times \cA} \int_{0}^1 x|r(x|s,a) - r'(x|s,a)| 
= \sup_{s,a \in \cS \times \cA} \E[|r(s,a)-r'(s,a)|].
\end{align}
\end{rmk}

\paragraph{Proof of $d_{\op{surr}} \le d_{l_1}$. } ~\\
By taking the log, dividing by $\log(1/\err)$, and taking the $\limsup$ of both sides of this inequality, we can infer the second statement of the lemma:
\begin{align}
\begin{split}
d_{\op{surr}} 
&= \limsup_{\err\to 0} \frac{\log(K_{\op{surr}}(\err))}{\log(1/\err)} \\
&\le \sum_{h=1}^H \limsup_{\err \to 0} \frac{\log(L_h^P(\err/(2H)^2))}{\log(\frac{1}{\err/(2H)^2})} \cdot \frac{1}{1+\log((2H)^2)/\log(\err/(2H)^2)} + \\
&\quad\sum_{h=1}^H \limsup_{\err \to 0} \frac{\log(L_h^R(\err/(4H)))}{\log(\frac{1}{\err/(2H)^2})} \cdot \frac{1}{1+\log(4H)/\log(\err/(4H))} \\
&= \sum_h d_{l_1,h}^P + \sum_h d_{l_1,h}^R 
= d_{l_1}.
\end{split}
\end{align}

\begin{fact}\label{fact:surr_l_1_estimation_bayesian}
    We separate the statement proved in \cref{eq:bound_surr_l_1_transitions,eq:bound_surr_l_1_rewards} as fact, useful for future use: for all $\cE,\cE' \in \Theta$,
    \begin{align}
        V_{1,\pi^*_\cE}^{\cE}(s_1^\ell)-V_{1, \pi^*_\cE}^{\cE'}(s_1^\ell)  \le 
        H \sum_{h=1}^H l_1(P_h^\cE, P_h^{\cE'}) + \sum_{h=1}^H l_1(r_h^\cE,r_h^{\cE'})\,.
    \end{align}
\end{fact}

\section{Proof of Lemma \ref{lem:surrogate_learning}}\label{appsec:lemma_surrogate_learning}
While we follow the proof of the same lemma in \cite[App. B.1]{hao2022regret}, we will need to correct some mistakes. Let us restate the equation of the statement:
\begin{lem*}
For any $\err-$value partition and any $\ell\in[L]$, there are random environments $\tilde \cE^*_\ell\in \Theta$ with their laws only depending on $\zeta,\cD_\ell$, such that
\begin{align}\label{eqn:regret_E_vs_surrogate_app}
    \mathbb E_{\ell}\left[V_{1, \pi^*_\cE}^\cE(s_1^\ell)-V_{1, \pi^\ell_{\ts}}^\cE(s_1^\ell)\right]-\mathbb E_{\ell}\left[V_{1,\pi^*_\cE}^{\tilde \cE_\ell^*}(s_1^\ell)-V_{1,\pi^\ell_{\ts}}^{\tilde \cE_\ell^*}(s_1^\ell)\right]\leq \varepsilon\,.
\end{align}  
\end{lem*}
\begin{proof}
Assume a partition $\Theta_k^\err$ satisfying \cref{dfn:err_value_partition_surrogate} with error $\err$ exists. Let  $\cE_\ell \sim P(\cdot | \cD_\ell)$. We want to decompose $\mathbb E_\ell\left[V_{1,\pi^\ell_\ts}^{\cE_\ell}(s_1^\ell)\big|\cE_\ell\in\Theta_k^\err\right]$, where the expectation is over all $\cE_\ell\in\Theta_k^\err$ and all $\pi^\ell_\ts = \pi^*_{\cE'_\ell}$ , where $\cE'_\ell \in \Theta$, with $\cE'_\ell$ independent of $\cE_\ell$. We decompose this by writing the expectation only over the former.
\begin{equation}\label{eq:surrogate_learning_different_equation}
\begin{split}
    \mathbb E_\ell\left[V_{1,\pi^\ell_\ts}^{\cE_\ell}(s_1^\ell)\big|\cE_\ell\in\Theta_k^\err\right] &= \int_{\cE_0\in\Theta_k^\err}\mathbb P\left(\cE_\ell=\cE_0|\cE_\ell\in\Theta_k^\err\right)\mathbb E_\ell\left[V_{1,\pi^\ell_\ts}^{\cE_0}(s_1^\ell)\big|\cE_\ell\in\Theta_k^\err\right]\d\rho(\cE_0)\\
    &=\int_{\cE_0\in\Theta_k^\err}\mathbb P\left(\cE_\ell=\cE_0|\cE_\ell\in\Theta_k^\err\right)\mathbb E_\ell\left[V_{1,\pi^\ell_\ts}^{\cE_0}(s_1^\ell)\right] \d\rho(\cE_0)\,,
\end{split}
\end{equation}
where the last equation is due to the independence between $\cE_\ell$ and $\cE_\ell'$.

We would like to find some $\tilde \cE_{k,\ell}^*$ such that its corresponding expected value is smaller than the integral above, i.e.
\begin{align}\label{eq:tilde_E_k_l_dfn}
    \mathbb E_\ell\left[V_{1,\pi^\ell_\ts}^{\tilde \cE_{k,\ell}^* }(s_1^\ell)\right] \le \int_{\cE_0\in\Theta_k^\err}\mathbb P\left(\cE_\ell=\cE_0|\cE_\ell\in\Theta_k^\err\right)\mathbb E_\ell\left[V_{1,\pi^\ell_\ts}^{\cE_0}(s_1^\ell)\right] \d\rho(\cE_0)
\end{align}
We set values of this random variable as $\tilde\cE^*_{k,\ell} = \E_\ell[\cE|\cE \in \Theta_k^\err], \forall k \in [K_{\op{surr}}(\err)]$. In other words, the posterior mean over $\Theta_k^\err$. Now, we can define a new random environment $\tilde\cE_\ell^*$  which prior can be easily computed as $\mathbb P_\ell\left(\tilde \cE_\ell^*=\tilde \cE_{k,\ell}^*\right) = \mathbb P_\ell(\cE \in \Theta_k^\err)$ and $\zeta(\tilde\cE^*_{k,\ell}) = k$. Note the law of $\tilde \cE_\ell^*$ only depends on $\zeta\in [K]$ and $\cD_\ell$, and conditional on $\zeta$, $\tilde \cE_\ell^*$ is independent of $\cE$, as desired. While this definition of $\tilde\cE^*_\ell$ may not be in $\Theta$, that is fine as we do not use this condition in our proof (see also \cref{rmk:tilde_E_can_be_convex_combination}). 

Notice that as a result, the overall posterior mean of $\bar\cE^*_\ell := \E_\ell[\tilde\cE^*_\ell]$ coincides with $\bar \cE_\ell$. Also note that by this definition of $\tilde\cE^*_{k,\ell}$, and the independence over layers even after conditioning on $\Theta^\err_k$, the mean $\E_\ell[V_{1,\pi^\ell_{\ts}}^{\tilde \cE_{k,\ell}^*}(s_1^\ell)]$ is in fact equal to the mean on the right hand side of \cref{eq:tilde_E_k_l_dfn}, thus our construction is sharp in that it satisfies this inequality with an equality.

Now, we are ready to make the connection with \cref{eqn:regret_E_vs_surrogate}, showing that solving the surrogate environment problem is `almost the same' as solving the original problem, up to some $\err$. Integrating over the different values of $k$ with prior $\mbbP_\ell(\cE \in \Theta_k^\err)$,
\begin{align}
    \E_\ell\left[V_{1,\pi^\ell_\ts}^{\tilde \cE^*_\ell }(s_1^\ell)\right] \le \mathbb E_\ell\left[V_{1,\pi^\ell_\ts}^{\cE}(s_1^\ell)\right] \implies \E_\ell\left[V_{1,\pi^\ell_\ts}^{\tilde \cE^*_\ell }(s_1^\ell)\right] - \mathbb E_\ell\left[V_{1,\pi^\ell_\ts}^{\cE}(s_1^\ell)\right] \le 0
\end{align}
Lastly, by the partition property,
\begin{align}\label{eq:only_use_of_err_value}
    \mathbb E_\ell\left[V_{1,\pi^*_\cE}^\cE(s_1^\ell)-V_{1,\pi^*_\cE}^{\tilde \cE_\ell^*}(s_1^\ell)\right]\leq \err\,.
\end{align}
where in the expectation above, $\zeta(\tilde\cE_\ell^*)=\zeta(\cE)$, i.e., $\cE$ and $\tilde \cE^*_\ell$ are clearly independent only after conditioning on $\zeta=k$, as also required in \cref{eqn:regret_E_vs_surrogate}. Adding the above two inequalities gives
\begin{align}
    \mathbb E_\ell\left[V_{1,\pi^*_\cE}^\cE(s_1^\ell)-V_{1,\pi^\ell_\ts}^{\cE}(s_1^\ell)\right] - \E_\ell\left[V_{1,\pi^*_\cE}^{\tilde \cE_\ell^*}(s_1^\ell) - V_{1,\pi^\ell_\ts}^{\tilde \cE^*_\ell }(s_1^\ell)\right]\leq \err
\end{align}
finishing the proof. Notice that \cref{eq:only_use_of_err_value} is the only place where the property of $\err-$value partitioning is used.
\end{proof}
\begin{rmk}\label{rmk:tilde_E_can_be_convex_combination}
We note that if $\tilde \cE^*_\ell$ is any convex combination of environments in $\Theta_{\zeta(\cE)}$, then by the $\err-$value partition property $\mathbb E_{\ell}[V_{1, \pi^*_\cE}^\cE(s_1^\ell) -  V_{1,\pi^*_\cE}^{\tilde \cE_\ell^*}(s_1^\ell)] \le \err$. 
The distinct property that surrogate environments satisfy is
\begin{align}\label{eq:tilde_minimize_TS}
    \E_{\ell} [V_{1,\pi^\ell_{\ts}}^{\tilde \cE_{k,\ell}^*}(s_1^\ell)] \le \E_{\ell} [V_{1, \pi^\ell_{\ts}}^\cE(s_1^\ell)| \cE \in \Theta_k^\err] , \ \forall k \in [K].
\end{align}
and therefore, satisfy \cref{eqn:regret_E_vs_surrogate_app}.
\end{rmk}

\begin{rmk}
    For the purpose of fully addressing the issue present in \cite{hao2022regret} in proving the above lemma, we show an alternative construction which does not assume a layered $\err-$value partition. One can always find a decreasing sequence $ \mathbb E_\ell\left[V_{1,\pi^\ell_\ts}^{\cE^1 }(s_1^\ell)\right] \ge \mathbb E_\ell\left[V_{1,\pi^\ell_\ts}^{\cE^2 }(s_1^\ell)\right] \ge \ldots $ with limit $ \inf_{\cE_0 \in \Theta_k^\err}\mathbb E_\ell\left[V_{1,\pi^\ell_\ts}^{\cE_0 }(s_1^\ell)\right] $. Then we claim there exists some $J$ such that the above inequality is true for $\tilde \cE_{k,\ell}^* = \cE^J$. Otherwise, we have
\begin{align}
    \mathbb E_\ell\left[V_{1,\pi^\ell_\ts}^{\cE^i }(s_1^\ell)\right] > \int_{\cE_0\in\Theta_k^\err}\mathbb P\left(\cE_\ell=\cE_0|\cE_\ell\in\Theta_k^\err\right)\mathbb E_\ell\left[V_{1,\pi^\ell_\ts}^{\cE_0}(s_1^\ell)\right] \d\rho(\cE_0)
\end{align}
for all $i$. Taking the limit $i\to \infty$, we get
\begin{align}
    \inf_{\cE_0 \in \Theta_k^\err}\mathbb E_\ell\left[V_{1,\pi^\ell_\ts}^{\cE_0 }(s_1^\ell)\right] \ge \int_{\cE_0\in\Theta_k^\err}\mathbb P\left(\cE_\ell=\cE_0|\cE_\ell\in\Theta_k^\err\right)\mathbb E_\ell\left[V_{1,\pi^\ell_\ts}^{\cE_0}(s_1^\ell)\right] \d\rho(\cE_0)
\end{align}
which can only be true if $\E_\ell\left[V_{1,\pi^\ell_\ts}^{\cE_0}(s_1^\ell)\right] = \inf_{\cE_0 \in \Theta_k^\err}\mathbb E_\ell\left[V_{1,\pi^\ell_\ts}^{\cE_0 }(s_1^\ell)\right]$ almost everywhere. Any $\cE_0$ satisfying this equality would also satisfy our requirement for $\tilde \cE_{k,\ell}^*$. It is important to note that $\tilde \cE_{k,\ell}^*$ depends on both $k$ and $\cD_\ell$, as mentioned in the lemma's statement. The finishes the construction of $\tilde \cE_\ell^*$ and the rest follows similar to the above proof.
\end{rmk}
\begin{rmk}
[Incorrect proof by \cite{hao2022regret} of the lemma]
\cref{eq:surrogate_learning_different_equation} is different from what appears in \cite[App. B.1]{hao2022regret}, where we have corrected for the abuse of notation of $\cE$ occurring in e.g. ``$\mathbb P\left(\cE_\ell=\cE|\cE_\ell\in\Theta_k^\err\right)\mathbb E_\ell\left[V_{1,\pi^*_\cE}^{\cE}(s_1^\ell)\right]$''. We further note the use of summation over $\cE\in \Theta_k^\err$ in their equation, instead of an integral which is required even in the tabular case. This is not easily fixed by just replacing sum with integral, since the application of \cite[Lemma D.1]{hao2022regret} used afterwards in their proof depends on having a finite sum. We mention this lemma below:

(Lemma 1 in \cite{dong2018information}) Let $\{a_i\}_{i=1}^N$ and $\{b_i\}_{i=1}^N$ be two sequences of real numbers, where $N<\infty$. Let $\{p_i\}_{i=1}^N$ be such that $p_i\geq 0$ for all $i$ and $\sum_{i=1}^N p_i=1$. Then there exists indices $j,k\in[N]$ and $r\in[0, 1]$ such that
\begin{equation*}
    ra_j+(1-r)a_k\leq \sum_{i=1}^Na_i p_i, rb_j+(1-r)b_k\leq \sum_{i=1}^Lb_i p_i\,.
\end{equation*}

The application of this lemma is key to the proof as the authors cite it to find the right surrogate environments that satisfy \cref{eq:tilde_minimize_TS}. However, even in this application, it is not clear what exactly is the `second' set of numbers ($b_j$ above), as is required for that Lemma to be a nontrivial result, which makes its usage even more questionable.
\end{rmk}

\section{Proof of Theorem \ref{thm:gen_TS_bd}}\label{appsec:generic_bound}

We restate the theorem for ease of reference as it contains multiple statements.
\begin{thm*}
Given a Bayesian RL problem, we have
\begin{align}
\BR_L(\pi_\ts) \leq \inf_{\err>0}\left(\lambda\sqrt{\log(K_{\op{surr}}(\err))T} + L\err\right) + T_0
\end{align}
where $T_0$ does not depend on $T$. This can be further upper bounded by
\begin{align}\label{eq:gen_l_1_regret_bd_app}
    \BR_L(\pi_\ts)\le \widetilde{O}(\lambda\sqrt{d_{l_1}T})\,.
\end{align}
for large enough $T$. Given a homogeneous $l_1$ dimension $d_{\op{hom}} = d_{l_1,h}, \forall h$,  this simplifies to
\begin{align}
    \BR_L(\pi_\ts)\leq \widetilde{O}(\lambda\sqrt{Hd_{\op{hom}}T})\,.
\end{align}
\end{thm*}

\begin{proof}
The proof starts by employing the surrogate environment learning bound from \cref{appsec:lemma_surrogate_learning}:
\begin{align}\label{eq:main_property_surr_env_app_2}
    \mathbb E_\ell\left[V_{1,\pi^*_\cE}^{\cE}(s_1^\ell)-V_{1,\pi^\ell_{\ts}}^\cE(s_1^\ell)\right]-\err\leq \mathbb E_{\ell}\left[V_{1,\pi^*_\cE}^{\tilde \cE_\ell^*}(s_1^\ell)-V_{1,\pi^\ell_{\ts}}^{\tilde \cE_\ell^*}(s_1^\ell)\right]\,.
\end{align}
We have
\begin{align}\label{eq:bayes_regret}
    \begin{split}
    \BR_L(\pi_{\ts}) &= \sum_{\ell=1}^L\mathbb E\left[\mathbb E_\ell\left[V_{1,\pi^*_\cE}^{\cE}(s_1^\ell)-V_{1,\pi^\ell_{\ts}}^\cE(s_1^\ell)\right]\right]\\
    &= \sum_{\ell=1}^L\mathbb E\left[\mathbb E_\ell\left[V_{1,\pi^*_\cE}^{\cE}(s_1^\ell)-V_{1,\pi^\ell_{\ts}}^\cE(s_1^\ell)\right]-\err\right]+L\err \\   
     &= \sum_{\ell=1}^L\mathbb E\left[\mathbb E_{\ell}\left[V_{1,\pi^*_\cE}^{\tilde \cE_\ell^*}(s_1^\ell)-V_{1,\pi^\ell_{\ts}}^{\tilde \cE_\ell^*}(s_1^\ell)\right]\right]+L\err
    \end{split}
\end{align}

Given that $\pi^\ell_\ts$ is independent from $\tilde\cE^*_\ell$ and the independence of the latter's prior over different layers (due to the layered $\err-$value partition), we observe that $\E_\ell[V_{1,\pi^\ell_{\ts}}^{\tilde \cE_\ell^*}(s_1^\ell)] =\E_\ell[V_{1,\pi^\ell_{\ts}}^{\bar\cE_\ell^*}(s_1^\ell)] $. Again, since $\pi^\ell_\ts$ is also independent from $\bar\cE^*_\ell$, and that $\pi^\ell_\ts$ and $\cE$ have the same laws conditional on $\cD_\ell$, we can rewrite the latter as $\E_\ell[V_{1,\pi^*_\cE}^{\bar\cE_\ell^*}(s_1^\ell)]$. This comes with the obvious note that $\bar\cE_\ell^*$ and $\cE$ are independent, in contrast to $\tilde\cE^*_\ell$  and $\cE$ that are dependent through $\zeta$. This allows us to rewrite 
 \begin{align}
     \begin{split}
    \mathbb E_\ell\left[V_{1,\pi^*_\cE}^{\tilde \cE_\ell^*}(s_1^\ell)-V_{1,\pi^\ell_{\ts}}^{\tilde \cE_\ell^*}(s_1^\ell)\right]&=
          \mathbb E_\ell\left[V_{1,\pi^*_\cE}^{\tilde \cE_\ell^*}(s_1^\ell)-V_{1,\pi^*_\cE}^{\bar \cE_\ell^*}(s_1^\ell)\right] 
     \end{split}
 \end{align} 
Due to our construction of $\bar\cE^*_\ell$ in \cref{appsec:lemma_surrogate_learning}, we can substitute $\bar\cE^*_\ell = \bar\cE_\ell = \E_\ell[\cE]$. 
Using \cref{lem:D_3_Tor_generalized_full_bayesian}, we can rewrite the above mean as 
\begin{align}\label{eq:d_3_application}
\begin{split}
= \sum_{h=1}^H\mathbb E_{\ell} & \left[\mathbb E_{\pi^*_\cE}^{\bar\cE_\ell}\left[
    \E_{(s', r') \sim (P_h^{\tilde \cE_\ell^*} \otimes r_h^{\tilde \cE_\ell^*})(\cdot|s_h,a_h)}[r' + V_{h+1,\pi^*_\cE}^{ \tilde \cE_\ell^*}(s')] 
\right.\right. \\
&\qquad\qquad\left.\left.
- \E_{(s', r') \sim (P_h^{\bar\cE_\ell} \otimes r_h^{\bar\cE_\ell})(\cdot|s_h,a_h)}[r' + V_{h+1,\pi^*_\cE}^{ \tilde \cE_\ell^*}(s')]
\right]\right]\,.
\end{split}
\end{align}
Denoting 
\begin{align}
\Delta_h^{\tilde \cE_\ell^*}(s_h,a_h)
&:= \E_{(s', r') \sim (P_h^{\tilde \cE_\ell^*} \otimes r_h^{\tilde \cE_\ell^*})(\cdot|s_h,a_h)}[r' + V_{h+1,\pi^*_\cE}^{ \tilde \cE_\ell^*}(s')] \nonumber \\
&\quad- \E_{(s', r') \sim (P_h^{\bar\cE_\ell} \otimes r_h^{\bar\cE_\ell})(\cdot|s_h,a_h)}[r' + V_{h+1,\pi^*_\cE}^{ \tilde \cE_\ell^*}(s')],
\end{align}
we have
 \begin{align}\label{eq:multiply_divide_d}
      = \sum_{h=1}^H\mathbb E_{\ell}\left[ \int_{s,a} d_{h,\pi^*_\cE}^{\bar\cE_\ell}(s,a) \Delta_h^{\tilde \cE_\ell^*}(s,a) \d \mu_{\cS\times\cA}\right]\,.
\end{align}
Let $\cB_\ell := \{ (s, a, h) \mid \mathbb E_\ell\left[ d_{h, \pi^*_\cE}^{\bar\cE_\ell}(s, a) \right] \neq 0 \}$
and let $\int_{(s, a, h)} := \sum_h \int_{(s, a)}$ denote the integral over the space $[H] \times \cS \times \cA$ where we use the product of counting measure on $[H]$ and $\mu_{\cS \times \cA}$.
We apply Cauchy-Schwarz inequality using the similar technique in \cite[App. A.2]{hao2022regret}, which we modify to include the value diameter (\cref{dfn:value_diameter}).
Since $\Delta_h^{\tilde\cE^*_\ell}(s, a) \leq 2H$, we have
\begin{align*}
 &\mathbb E_\ell \left[ 
    \int_{(s, a, h)} d_{h, \pi^*}^{\bar\cE_\ell}(s, a) \Delta_h^{\tilde\cE^*_\ell}(s, a)
\right] \\
&=\mathbb E_\ell \left[ 
    \int_{(s, a, h) \notin \cB_\ell} d_{h, \pi^*}^{\bar\cE_\ell}(s, a) \Delta_h^{\tilde\cE^*_\ell}(s, a)
\right]
+ \mathbb E_\ell \left[ 
    \int_{(s, a, h) \in \cB_\ell} d_{h, \pi^*}^{\bar\cE_\ell}(s, a) \Delta_h^{\tilde\cE^*_\ell}(s, a)
\right] \\
&\leq 2H \mathbb E_\ell \left[ 
    \int_{(s, a, h) \notin \cB_\ell} d_{h, \pi^*}^{\bar\cE_\ell}(s, a)
\right]
+ \mathbb E_\ell \left[ 
    \int_{(s, a, h) \in \cB_\ell} d_{h, \pi^*}^{\bar\cE_\ell}(s, a) \Delta_h^{\tilde\cE^*_\ell}(s, a)
\right] \\
&= \mathbb E_\ell \left[ 
    \int_{(s, a, h) \in \cB_\ell} d_{h, \pi^*}^{\bar\cE_\ell}(s, a) \Delta_h^{\tilde\cE^*_\ell}(s, a)
\right] \\
&= \mathbb E_\ell \left[ 
    \int_{(s, a, h) \in \cB_\ell} 
    \frac
        {\lambda_\cE d_{h, \pi^*}^{\bar\cE_\ell}(s, a)}
        { \mathbb E_\ell\left[ d_{h, \pi^*}^{\bar\cE_\ell}(s, a) \right]^{1/2} }
    \mathbb E_\ell\left[ d_{h, \pi^*}^{\bar\cE_\ell}(s, a) \right]^{1/2}
    \frac{\Delta_h^{\tilde\cE^*_\ell}(s, a)}{\lambda_\cE}
\right] \\
&\leq \left( \mathbb E_\ell \left[ 
    \int_{(s, a, h) \in \cB_\ell} 
    \frac
        { (\lambda_\cE d_{h, \pi^*}^{\bar\cE_\ell}(s, a) )^2}
        { \mathbb E_\ell\left[ d_{h, \pi^*}^{\bar\cE_\ell}(s, a) \right] }
\right] \right)^{1/2} \\
&\qquad\qquad\qquad\qquad\left(
\mathbb E_\ell \left[ 
    \int_{(s, a, h) \in \cB_\ell} 
    \mathbb E_\ell\left[ d_{h, \pi^*}^{\bar\cE_\ell}(s, a) \right]
    (\frac{\Delta_h^{\tilde\cE^*_\ell}(s, a)}{\lambda_\cE})^2
\right]\right)^{1/2} \\
&\leq \left(\int_{(s, a, h) \in \cB_\ell} \frac
    { \mathbb E_\ell\left[ (\lambda_\cE d_{h, \pi^*}^{\bar\cE_\ell}(s, a) )^2 \right]}
    { \mathbb E_\ell\left[ d_{h, \pi^*}^{\bar\cE_\ell}(s, a) \right] }
\right)^{1/2}
\left(
\mathbb E_\ell \left[ 
    \int_{(s, a, h)} 
    \mathbb E_\ell\left[ d_{h, \pi^*}^{\bar\cE_\ell}(s, a) \right]
    (\frac{\Delta_h^{\tilde\cE^*_\ell}(s, a)}{\lambda_\cE})^2
\right]\right)^{1/2} \\
&= \sqrt{\cT^\ell \cdot \cI^\ell},
\end{align*}
where we used $\cT^\ell$ and $\cI^\ell$ to denote the first and the second term respectively.
Note that the total regret of each episode is at most $H$.
Therefore, going back to the Bayesian regret formulationa and using Cauchy-Schwarz, we get
\begin{align}\label{eq:bayes_regret_cauchy}
\BR_L(\pi_{\ts}) 
&\le \E\left[
    \sum_{\ell=L_0+1}^L \sqrt{\cT^\ell \cdot \cI^\ell}
\right] + L_0 H + L \varepsilon \\
&\le \E\left[
    (\sum_{\ell=L_0+1}^L \cT^\ell)^{1/2}
    \cdot 
    (\sum_{\ell=L_0+1}^L \cI^\ell)^{1/2} 
\right]  + L_0 H + L \varepsilon \\
&\le \sqrt{
    \E\left[ \sum_{\ell=L_0+1}^L \cT^\ell \right]
    \cdot
    \E\left[ \sum_{\ell=L_0+1}^L \cI^\ell \right]
} + L_0 H + L \varepsilon \\
&\le \sqrt{
    L \left( \sup_{L_0+1 \leq \ell \leq L} \E\left[ \cT^\ell \right] \right)
    \cdot
    \E\left[ \sum_{\ell=1}^L \cI^\ell \right]
} + L_0 H + L \varepsilon,
\end{align}
for every $0 \leq L_0 < L$.
We estimate each term separately in \cref{appsec:estimation_Gamma_ell} and \cref{appssec:estimation_I_ell}. 
\begin{rmk}
While the spirit of the argument, in applying surrogate learning coupled with information ratio and Cauchy Schwarz is similar to \cite{hao2022regret}, the technical aspects are different for estimating $\sum_\ell \cI^\ell$, and more importantly, the entire analysis is different for $\sum_\ell \cT^\ell$, as can be observed in what follows.
\end{rmk}
We gather the results to finish the proof. 
In \cref{appssec:estimation_I_ell}, we show that  $\E[\sum_\ell\cI^\ell] \le \frac{1}{2}\log(K_{\op{surr}}(\err))$.

In \cref{appsec:estimation_Gamma_ell}, we show that $\limsup \E\left[ \cT^\ell \right]$ is bounded by $\lambda^2 H$.
Thus $\sup_{L_0+1 \leq \ell \leq L} \E\left[ \cT^\ell \right]
 \le 2\lambda^2 H$ for large enough $L_0 > 0$.
Hence, 
\begin{align}
    \BR_L(\pi_\ts)\leq \lambda\sqrt{\log(K_{\op{surr}}(\err))T} + L\err + T_0
\end{align}
for all $\err>0$, where $T_0 = L_0 H$.
Taking the infimum over $\err$ gives the desired regret bound.
The next statement of the theorem in \cref{eq:gen_l_1_regret_bd_app} is an application of \cref{lem:l_1_dim_surrogate} followed by selecting $\err = 1/L$. Notice that due to nonzero $\err$ effects, polylogarithmic terms in $H,L$ are picked up when comparing $d_{l_1}$ with $d_{l_1}(\err)$. More precisely, according to \cref{lem:l_1_dim_surrogate}, one compares $\log(K_{\op{surr}}(\err))$ with $\sum_h \log(L_h^P(\err/(2H)^2))+\log(L_h^R(\err/(4H)))$. From the definition of $d_{l_1,h}^P$ we have $\log(L_h^P(\err/(2H)^2)) \sim O(d_{l_1,h}^P \log((2H)^2/\err))$, so this includes a logarithmic factor of $H^2$, and choosing $\err=1/L$ means a logarithmic factor of $H^2L$. A similar argument can be made for $L_h^R$.

Finally, the last statement of the theorem follows by definition: $d_{l_1}L =(\sum_h d_{l_1,h})L =  d_{\op{hom}}HL= d_{\op{hom}}T$. This finishes the proof of the main theorem.
\end{proof}

\begin{rmk}
Equations similar to \cref{eq:bayes_regret} can be found in \cite[App. B.2]{hao2022regret}. However, a small correction must be made to their derivation. The authors first apply Cauchy-Schwarz and then take the square of \cref{eq:main_property_surr_env_app_2} to replace the original regret with the surrogate regret. As seen later in the proof, we do it in the opposite order, because in (\ref{eq:main_property_surr_env_app_2}), the left side may be negative, so we can not assume the square of that estimation to be also correct. While outside the focus of this paper, we note that a side-effect of this correction is another one to their definition of surrogate-IDS in \cite[Eq. (4.3)]{hao2022regret}, wherein minimization should be over the square root of their information ratio.
\end{rmk}

\section{Bounding \texorpdfstring{$\E[\sum_\ell\cI^\ell]$}{I} }\label{appssec:estimation_I_ell}
We start by proving $\cI^\ell \le \frac{1}{2}\mbbI^{\pi^\ell_\ts}_\ell(\tilde\cE^*_\ell;\cH_{\ell,H}), \forall \ell \in [L]$. We recall that TS property implies $\E_\ell[d_{h,\pi^*_\cE}^{\bar\cE_\ell}(s,a)] = \E_\ell[d_{h,\pi^\ell_\ts}^{\bar\cE_\ell}(s,a)]$. So,
\begin{align}
    \cI^\ell &= \sum_{h=1}^H\mathbb E_{\ell}\left[ \int_{s,a}  \E_\ell[d_{h,\pi^*_\cE}^{\bar\cE_\ell}(s,a)]\frac{\Delta_h^{\tilde \cE_\ell^*}(s,a)^2}{\lambda_\cE^2} \d \mu_{\cS\times\cA}\right] \\
    &= \sum_{h=1}^H\mathbb E_{\ell}\left[ \int_{s,a}  \E_\ell[d_{h,\pi^\ell_\ts}^{\bar\cE_\ell}(s,a)]\frac{\Delta_h^{\tilde \cE_\ell^*}(s,a)^2}{\lambda_\cE^2} \d \mu_{\cS\times\cA}\right]\,.
\end{align}
Next, we swap the two integrals, one represented by $\E_\ell$ and the one over $s,a$, 
\begin{align}
    =\sum_{h=1}^H\int_{s,a} \E_\ell[d_{h,\pi^\ell_\ts}^{\bar\cE_\ell}(s,a)] \mathbb  E_{\ell}\left[  \frac{\Delta_h^{\tilde \cE_\ell^*}(s,a)^2}{\lambda_\cE^2} \right]\d \mu_{\cS\times\cA}
\end{align}
Note that, given $\cD_\ell$, $\Delta_h^{\tilde\cE^*_\ell}$ is independent of $d_{h,\pi^\ell_\ts}^{\bar\cE_\ell}(s,a)$. Therefore, using the identity $\E[XY]=\E[X]\E[Y]$ for two independent random variables, and swapping back the two integrals,
\begin{align}
    = \sum_{h=1}^H\mathbb  E_{\ell}\left[ \int_{s,a}  d_{h,\pi^\ell_\ts}^{\bar\cE_\ell}(s,a)\frac{\Delta_h^{\tilde \cE_\ell^*}(s,a)^2}{\lambda_\cE^2} \d \mu_{\cS\times\cA}\right] = \sum_{h=1}^H\mathbb  E_{\ell}\left[\E_{\pi^\ell_\ts}^{\bar\cE_\ell}\left[ \frac{\Delta_h^{\tilde \cE_\ell^*}(s,a)^2}{\lambda_\cE^2} \right]\right]
\end{align}
Finally, notice that $\frac{\Delta_h^{\tilde \cE_\ell^*}(s,a)^2}{\lambda_\cE^2}$ can be estimated by Pinsker's inequality (\cref{lem:pinsker}) as
\begin{align}
\begin{split}\label{eq:normalization_lambda_cE}
&\frac{\Delta_h^{\tilde \cE_\ell^*}(s,a)^2}{\lambda_\cE^2} \\
&\quad= \left( \E_{(s', r') \sim (P_h^{\tilde \cE_\ell^*} \otimes r_h^{\tilde \cE_\ell^*})(\cdot|s_h,a_h)}\left[\frac{r' + V_{h+1,\pi^*_\cE}^{ \tilde \cE_\ell^*}(s')}{\lambda_\cE}\right] \right. \nonumber \\
&\quad\qquad \left. - \E_{(s', r') \sim (P_h^{\bar\cE_\ell} \otimes r_h^{\bar\cE_\ell})(\cdot|s_h,a_h)}\left[ \frac{r' + V_{h+1,\pi^*_\cE}^{ \tilde \cE_\ell^*}(s')}{\lambda_\cE}\right] \right)^2 \\
&\quad= \left(\E_{(s', r') \sim (P_h^{\tilde \cE_\ell^*} \otimes r_h^{\tilde \cE_\ell^*})(\cdot|s_h,a_h)}\left[ \frac{ r' +  V_{h+1,\pi^*_\cE}^{ \tilde \cE_\ell^*}(s') - r^{\op{inf}}_h(s_h, a_h) - \inf_s V_{h+1,\pi^*_\cE}^{ \tilde \cE_\ell^*}(s)}{\lambda_\cE} \right] \right. \\
&\quad\qquad \left. - \E_{(s', r') \sim (P_h^{\bar\cE_\ell} \otimes r_h^{\bar\cE_\ell})(\cdot|s_h,a_h)}\left[ \frac{r' + V_{h+1,\pi^*_\cE}^{ \tilde \cE_\ell^*}(s') - r^{\op{inf}}_h(s_h, a_h) - \inf_s V_{h+1,\pi^*_\cE}^{ \tilde \cE_\ell^*}(s)}{\lambda_\cE} \right]\right)^2 \\
&\quad\leq \frac{1}{2}\DKL\left( (P_h^{\tilde \cE_\ell^*} \otimes r_h^{\tilde \cE_\ell^*})(\cdot|s_h,a_h)|| (P_h^{ \bar\cE_\ell} \otimes r_h^{ \bar\cE_\ell})(\cdot|s_h,a_h) \right)
\end{split}
\end{align}
where we note the trick of adding and subtracting the constant term $r^{\op{inf}}_h(s_h, a_h) + \inf_s V_{h+1,\pi^*_\cE}^{ \tilde \cE_\ell^*}(s)$ to the expected values in order to make the expression inside between zero and $\lambda_\cE$.
This enables the application of Pinsker's inequality which requires the random variable $X$ in $(\E_P[X]-\E_Q[X])^2 \le \frac{1}{2}\DKL(P||Q)$ to be smaller than one. 
Therefore 
\begin{align}
\cI^\ell \le \frac{1}{2}\sum_{h=1}^H \E_\ell\left[\E_{\pi^\ell_\ts}^{\bar\cE_\ell}\left[\DKL \left(
(P_h^{\tilde \cE_\ell^*} \otimes r_h^{\tilde \cE_\ell^*})(\cdot|s_h,a_h)|| (P_h^{ \bar\cE_\ell} \otimes r_h^{ \bar\cE_\ell})(\cdot|s_h,a_h)
\right)\right]\right].
\end{align}
Lastly, we use \cref{lem:mutual_information_rewrite}, wherein we show a fact similar to \cite[App. C.1.]{hao2022regret} but for $\tilde\cE^*_\ell$ instead of $\cE$, proving the above equals $\frac{1}{2}\mbbI_\ell^{\pi^\ell_\ts}(\tilde \cE^*_\ell;\cH_{\ell,H})$. 

Next, observe that $\mathbb I_\ell^{\pi^\ell_\ts}(\tilde \cE_\ell^*; \cH_{\ell, H}) \le \mathbb I_\ell^{\pi^\ell_\ts}(\zeta; \cH_{\ell, H})$. 
Indeed, recall that the definition of $\mbbI_\ell$ conditions on a $\cD_\ell$; and we know that conditional on a $\cD_\ell$, the surrogate environment $\tilde\cE^*_\ell$ is only dependent on $\zeta$ by construction, hence the data processing inequality applies and we have
\begin{align}\label{eq:I_bound_main}
\cI^\ell \le \frac{1}{2} \mathbb I_\ell^{\pi^\ell_\ts}(\zeta; \cH_{\ell, H}).
\end{align}
Next, we use the mutual information chain rule, observing that 
\begin{align}\label{eq:I_l_conditioned_on_history}
\mathbb E\left[\mathbb I_{\ell}^{\pi^\ell_\ts}(\zeta; \cH_{\ell,H})\right] 
= \mathbb I(\zeta; \cH_{\ell,H}| \cH_{\ell-1,H},\ldots,\cH_{1,H}),
\end{align}
and therefore
\begin{align}
\mathbb I\left(\zeta;\cD_{L+1}\right)
= \mathbb I\left(\zeta; \left(\cH_{1, H}, \ldots, \cH_{L, H}\right)\right) 
= \sum_{\ell=1}^L\mathbb E\left[\mathbb I_{\ell}^{\pi^\ell_\ts}(\zeta; \cH_{\ell,H})\right]\,.
\end{align}
Applied to the above, and noting that $\mathbb I\left(\zeta;\cD_{L+1}\right) \le H(\zeta) \le \log(K_{\op{surr}}(\err))$, this finishes our estimation of $\E[\sum_\ell\cI^\ell] \le \frac{1}{2}\log(K_{\op{surr}}(\err))$.

\section{Bounding \texorpdfstring{$\E[\cT^\ell]$}{T}}\label{appsec:estimation_Gamma_ell}
This is where we use analysis tools from posterior consistency. 
We focus on bounding
\begin{align}
\cT^\ell 
= \int_{(s, a, h) \in \cB_\ell} 
\frac
    { \mathbb E_\ell\left[ (\lambda_\cE d_{h, \pi^*_\cE}^{\bar\cE_\ell}(s, a) )^2 \right] }
    { \mathbb E_\ell\left[ d_{h, \pi^*_\cE}^{\bar\cE_\ell}(s, a) \right] }.
\end{align}
Note that we have
\begin{align*}
\mathbb E_\ell\left[ d_{h, \pi^*_\cE}^{\bar\cE_\ell}(s, a) \right]
= \mathbb E_\ell\left[ d_{h, \pi^*_\cE}^{\cE'}(s, a) \right],
\end{align*}
where $\cE'$ is sampled from the posterior $\mathbb P_\ell$ independent of $\cE$.
Similarly
\begin{align*}
\mathbb E_\ell\left[ (\lambda_\cE d_{h, \pi^*_\cE}^{\bar\cE_\ell}(s, a) )^2 \right]
&= (\mathbb E_\ell)_{\cE \sim \mathbb P_\ell}\left[ \lambda_\cE^2 \left( 
    (\mathbb E_\ell)_{\cE' \sim \mathbb P_\ell}\left[ d_{h, \pi^*_\cE}^{\cE'}(s, a) \right] 
\right)^2 \right] \\
&\leq (\mathbb E_\ell)_{\cE \sim \mathbb P_\ell}\left[ \lambda_\cE^2 
    (\mathbb E_\ell)_{\cE' \sim \mathbb P_\ell}\left[ d_{h, \pi^*_\cE}^{\cE'}(s, a)^2 \right]
\right] \\
&= \mathbb E_\ell \left[ (\lambda_\cE d_{h, \pi^*_\cE}^{\cE'}(s, a) )^2 \right].
\end{align*}
For any $\ell, s, a, h$ and $\cD_\ell$, we define
\begin{align}
g_\ell(s, a, h, \cD_\ell) := \frac
{ \mathbb E_\ell\left[ (\lambda_\cE d_{h, \pi^*_\cE}^{\cE'}(s, a) )^2 \right] }
{ \mathbb E_\ell\left[ d_{h, \pi^*_\cE}^{\cE'}(s, a) \right] },
\end{align}
whenever $(s, a, h) \in \cB_\ell$ and $g_\ell(s, a, h, \cD_\ell) := 0$ otherwise.
Clearly we have
\begin{align}
\cT^\ell \leq \int_{(s, a, h)} g_\ell(s, a, h, \cD_\ell).
\end{align}
Moreover, given our assumption on state action occupation density, we have $M_d := \sup_{s, a, h, \pi, \cE} d_{h, \pi}^{\cE}(s, a) < \infty$, which implies $g_\ell \leq M_d H^2 < \infty$ and $\cT^\ell \leq M_d H^3 < \infty$.
Let $\cE_0$ be the true environment. 
This means $\mathbb E [ \cdot | \cE_0] = \mathbb E_{\cD_\ell \sim \mathbb P(\cdot | \cE_0)}[ \cdot ]$.

According to Corollary~\ref{cor:doob-bounded}, we have
\begin{align}
\lim_{\ell \to \infty} 
\mathbb E_\ell\left[ d_{h, \pi^*_\cE}^{\cE'}(s, a) \right] 
&= d_{h, \pi^*_{\cE_0}}^{\cE_0}(s, a), \\
\lim_{\ell \to \infty} 
\mathbb E_\ell\left[ (\lambda_\cE d_{h, \pi^*_\cE}^{\cE'}(s, a) )^2 \right] 
&= (\lambda_{\cE_0} d_{h, \pi^*_{\cE_0}}^{\cE_0}(s, a) )^2.
\end{align}
According to Assumption~\ref{assumption:non-zero-state-action}, $d_{h, \pi^*_{\cE_0}}^{\cE_0}(s, a) \neq 0$ for almost every $\cE_0, s, a$ and $h$.
For any such values of $(\cE_0, s, a, h)$ and almost every $\cD_\ell$ sampled from true environment $\cE_0$, we conclude that
\begin{align}
\lim_{\ell \to \infty} g_\ell(s, a, h, \cD_\ell)
= \lambda_{\cE_0} ^2 d_{h, \pi^*_{\cE_0}}^{\cE_0}(s, a).
\end{align}
Therefore, using dominated convergence theorem, for almost every $\cE_0$ we have 
\begin{align}
\lim_{\ell \to \infty} 
\mathbb E [\cT^\ell | \cE_0] 
&\leq \lim_{\ell \to \infty} 
\mathbb E_{\cD_\ell \sim \mathbb P(\cdot | \cE_0)} \left[ 
    \int_{(s, a, h)} g_\ell(s, a, h, \cD_\ell) 
\right] \\
&= \lim_{\ell \to \infty} 
\int_{(s, a, h)} \mathbb E_{\cD_\ell \sim \mathbb P(\cdot | \cE_0)} \left[ 
    g_\ell(s, a, h, \cD_\ell) 
\right] \\
&= \int_{(s, a, h)} \mathbb E_{\cD_\ell \sim \mathbb P(\cdot | \cE_0)} \left[ 
    \lim_{\ell \to \infty} g_\ell(s, a, h, \cD_\ell) 
\right] \\
&= \int_{(s, a, h)} \mathbb E_{\cD_\ell \sim \mathbb P(\cdot | \cE_0)} \left[ 
    \lambda_{\cE_0}^2 d_{h, \pi^*_{\cE_0}}^{\cE_0}(s, a)
\right] \\
&= \int_{(s, a, h)}
    \lambda_{\cE_0}^2 d_{h, \pi^*_{\cE_0}}^{\cE_0}(s, a) \\
&= \lambda_{\cE_0}^2 H. \label{eq:T_l_freq_bound}
\end{align}
Therefore we may use dominated convergence theorem again to see that
\begin{align}
\lim_{\ell \to \infty} \mathbb E[ \cT^\ell ]
= \lim_{\ell \to \infty} 
\mathbb E[ \mathbb E [\cT^\ell | \cE_0] ]
= \mathbb E [ 
    \lim_{\ell \to \infty} \mathbb E [\cT^\ell | \cE_0] 
] 
\leq \mathbb E [ \lambda_{\cE_0}^2 H] 
= \mathbb E [\lambda_{\cE}^2] H=\lambda^2H.
\end{align}
Therefore, there exists $L_0 > 0$ such that $\mathbb E[ \cT^\ell ] \le 2\lambda^2 H$ for $\ell > L_0$.

\section{Proof of Corollary \ref{cor:linear_RL}}\label{appsec:linear_RL}
We restate the corollary below. We prove it in a more general case where the maps $\phi^P,\phi^R$ are time inhomogeneous, i.e. $\phi_h^P,\phi_h^R$. In addition, the dimension of their target space can also depend on $h$, i.e. we use $d_{f}^{P,h},d_{f}^{R,h}$, instead of just $d_f^P,d_f^R$. For the case where the dimensions are homogeneous, we use $d_f^{\op{hom}}$. This new notation impacts the statement as follows:
\begin{cor*}
    For a linear Bayesian RL, for large enough $T$,
    \begin{align}
    \BR_L(\pi_\ts)\leq \widetilde{O}(\lambda\sqrt{d_{l_1}^fT}).
    \end{align}    
     Given a linear Bayesian RL with finitely many states and homogeneous feature space dimension $d_f^{\op{hom}}$, we have $d_{l_1}^f \le 2d_f^{\op{hom}}HS$, yielding for large enough $T$,
    \begin{align}\label{eq:linear_RL_finite_state_app}
    \BR_L(\pi_\ts)\leq \widetilde{O}(\lambda\sqrt{Hd_f^{\op{hom}}ST}).
    \end{align}
    Given a mixture linear Bayesian RL, for large enough $T$,
\begin{align}
        \BR_L(\pi_\ts)\leq \widetilde{O}(\lambda\sqrt{MT})\,,
    \end{align}
\end{cor*}
The first statement follows from the generic bound, but we need to relate the $l_1-$dimension of the environment space to that of the feature maps space, where we recall the definitions $d_{l_1}^f=d_{l_1}^{P,f}+d_{l_1}^{R,f}$ as the sum of the $l_1-$dimensions of the feature map space $\{\psi_{h}^{P,\cE}\}_{\cE \in \Theta},\{\psi_{h}^{R,\cE}\}_{\cE \in \Theta}$ where the $l_1-$distance between feature maps is defined as $l_1(\psi_{h}^{\cE},\psi_h^{\cE'}) = \int_s \|\psi_{h}^{\cE}-\psi_h^{\cE'}\|_1\mu_\cS$. We shall use \cref{fact:surr_l_1_estimation_bayesian}:
\begin{align}
     V_{1,\pi^*_\cE}^{\cE}(s_1^\ell)-V_{1, \pi^*_\cE}^{\cE'}(s_1^\ell)  \le 
        H \sum_{h=1}^H l_1(P_h^\cE, P_h^{\cE'}) + \sum_{h=1}^H l_1(r_h^\cE,r_h^{\cE'})\,.
\end{align}
We estimate
\begin{align}
    \begin{split}
       l_1(P_h^\cE, P_h^{\cE'}) = \sup_{s,a} \int_{s'}|P_h^\cE(s'|s,a)-P_h^{\cE'}(s'|s,a)| = \sup_{s,a}\int_{s'} | \phi_h^{P}(s,a) \cdot (\psi_h^{P,\cE}(s') - \psi_h^{P,\cE}(s'))| \le\\
    \sup_{s,a}\int_{s'} \sum_{i=1}^{d_{f,h}^P} |\phi_h^{P}(s,a)_i(\psi_h^{P,\cE}(s') - \psi_h^{P,\cE}(s'))_i| 
    \end{split}
\end{align}
Since $\|\phi_h^{P}(s,a)\|_2 \le 1 \implies |\phi_h^{P}(s,a)_i| \le 1, \forall i \in [d]$. Therefore
\begin{align}
    l_1(P_h^\cE, P_h^{\cE'}) \le   \int_{s'} \|(\psi_h^{P,\cE}(s') - \psi_h^{P,\cE}(s'))\|_1 = l_1(\psi_h^{P,\cE},\psi_h^{P,\cE}).
\end{align}
The similar bound can be achieved for $l_1(r_h^\cE, r_h^{\cE'})$. As a result $d_{l_1} \le d_{l_1}^f$ and we get the first statement.

For \cref{eq:linear_RL_finite_state_app}, we note that if $S$ is a finite, then $\psi_h^{P,\cE}$ can be viewed as a $d_f^{P,h}\times S$ matrix, or simply a vector with dimension that size. Therefore, we can view our problem as asking for the asymptotics of the $\err-$covering number in $\mbbR^{d_f^{P,h}S}$. As long as the collection $\{\psi_h^{P,\cE}\}_{\cE}$ is within a finite ball, which they are (\cref{dfn:RL_linear}), the covering number is well-known to scale at most as $O((\frac{\log(C_\psi)}{\err})^{d_f^{P,h}S})$ where $C_\psi$ is the radius of that ball. Applying the similar argument for the rewards, we have $d_{l_1} \le d_{l_1}^f \le \sum (d_f^{P,h}+d_f^{R,h})S $, which equals $2Hd_f^{\op{hom}}S$ given a homogeneous feature space dimension $d_f^{\op{hom}} = d_f^{P,h}=d_f^{R,h}$, for all $h\in [H]$.

The finite mixtures statement is a straightforward generalization of the above, where every $\psi_h^{P,\cE},\psi_h^{R,\cE}$ is characterized with a finite $m_h^P,m_h^R$-dimensional vector instead of specifically being $d_f^{P,h}S,d_f^{R,h}S$-dimensional. We note
\begin{align}
    l_1(\psi_h^{P,\cE},\psi_h^{P,\cE}) = \int_{s'} \|(\psi_h^{P,\cE}(s') - \psi_h^{P,\cE}(s'))\|_1 \le  C_\Psi \|\bm{a}_h^{P,\cE} - \bm{a}_h^{P,\cE'}\|_1 
\end{align}
where $C_\Psi =\max_{1\le i\le m_h^P} \|\Psi_{h,i}^P\|_1$. So the same argument above applies, where we consider the collection of finite dimensional vectors $\{\bm{a}_h^{P,\cE}\}_\cE$ and the $\err-$covering number, and similarly for $\{\bm{a}_h^{R,\cE}\}_\cE$.

\subsection{Incorrect proof for the Bayesian regret bound of linear Bayesian RL with deterministic rewards}\label{appssec:incorrect_proof_linear_RL}
Here, we discuss the proof of \cite{hao2022regret} for linear RLs, and the mistakes in their argument. We start by citing the similar equations in \cite[App. B.4]{hao2022regret} in bounding the value difference with the feature maps difference:
\begin{quote}
For any $\cE_1, \cE_2\in\Theta_k$, [...]
\begin{align*}
&V_{1, \pi^*_{\cE_1}}^{\cE_1}(s_1) - V_{1,\pi^*_{\cE_1}}^{\cE_2}(s_1) \\
&\qquad= \sum_{h=1}^H\mathbb E_{\pi^*_{\cE_1}}^{ \cE_2}\left[P_h^{ \cE_1}(\cdot|s_h^\ell,a_h^\ell)^{\top}V_{h+1,\pi^*_{\cE_1}}^{ \cE_1}(\cdot) - P_h^{ \cE_2}(\cdot|s_h^\ell,a_h^\ell)^{\top}V_{h+1,\pi^*_{\cE_1}}^{ \cE_1}(\cdot)\right] \\
&\qquad= \sum_{h=1}^H\mathbb E_{\pi^*_{\cE_1}}^{ \cE_2}\left[\phi(s_h^\ell, a_h^\ell)^{\top}\sum_{s'} V_{h+1,\pi^*_{\cE_1}}^{ \cE_1}(s') \psi_h^{\cE_1}(s') \right. \\
&\qquad\qquad\qquad\qquad\qquad\left. - \phi(s_h^\ell, a_h^\ell)^{\top}\sum_{s'}V_{h+1,\pi^*_{\cE_1}}^{ \cE_1}(s') \psi_h^{\cE_2}(s')\right]\,,
\end{align*}
    [...] Moreover, since the value function is always bounded by $H$, we have 
\begin{align}
V_{1, \pi^*_{\cE_1}}^{\cE_1}(s_1) &- V_{1,\pi^*_{\cE_1}}^{\cE_2}(s_1) \nonumber \\
&= H \sum_{h=1}^H\mathbb E_{\pi^*_{\cE_1}}^{ \cE_2}\left[\phi(s_h^\ell, a_h^\ell)^{\top}\left(\sum_{s'} \psi_h^{\cE_1}(s')-\sum_{s'} \psi_h^{\cE_2}(s')\right)\right] \nonumber \\
&\leq H\sum_{h=1}^H\mathbb E_{\pi^*_{\cE_1}}^{ \cE_2}\left[\left\|\phi(s_h^\ell, a_h^\ell)\right\|_2\right]\left\|\sum_{s'}\psi_h^{\cE_1}(s')-\sum_{s'} \psi_h^{\cE_2}(s')\right\|_2 \nonumber \\
&\leq H\sum_{h=1}^H\left\|\sum_{s'} \psi_h^{\cE_1}(s')-\sum_{s'} \psi_h^{\cE_2}(s')\right\|_2\,.  \label{eq:the_wrong_estimate_tor}
\end{align}
\end{quote}
Clearly, in the first equation, the equality must be replaced by $\le$, and more importantly, given that the value function $V_{h+1,\pi^*_{\cE_1}}^{ \cE_1}(\cdot)$ has argument $\cdot = s'$, the  $l_2-$norm should be taken on the \textbf{inside} of the integral $\int_{s'} \|\phi(s_h^\ell, a_h^\ell)^{\top} \left( (\psi_h^{\cE_1}(s')-\psi_h^{\cE_2}(s')\right)\|_2$ (we note we also replaced the sum $\sum_{s'}$ with integral, as linear RLs could have infinitely many states). If our proposed correction were to be followed then the next equations would change to ones similar to ours except with an $l_2-$distance:
\begin{align}
    &\leq H\sum_{h=1}^H\mathbb E_{\pi^*_{\cE_1}}^{ \cE_2}\left[\left\|\phi(s_h^\ell, a_h^\ell)\right\|_2\right]\int_{s'}\left\|\psi_h^{\cE_1}(s')-\psi_h^{\cE_2}(s')\right\|_2\\
       &\leq H\sum_{h=1}^H\int_{s'}\left\|\psi_h^{\cE_1}(s')-\psi_h^{\cE_2}(s')\right\|_2
\end{align}
Otherwise, let us assume that the authors were correct, then we have managed to bound $V_{1, \pi^*_{\cE_1}}^{\cE_1}(s_1)-V_{1,\pi^*_{\cE_1}}^{\cE_2}(s_1) \le H\sum_{h=1}^H\left\|\sum_{s'} \psi_h^{\cE_1}(s')-\sum_{s'} \psi_h^{\cE_2}(s')\right\|_2$. We show that this is a bound by zero for an important subclass of linear RLs, i.e. all tabular RLs.

It is a well-known fact that tabular RLs can be viewed as linear RLs. The mapping works as follows. First let us enumerate the set $\{(s,a) \in \cS \times \cA\}$ by $1, \ldots, SA$. Call this assignment $N(s,a) \in [SA]$. Then define $\phi(s,a) = e_{N(s,a)} \in \mbbR^{SA}$, which is the Euclidean basis state on axis $N(s,a)$. Let $\psi_h^\cE(s') = (P_h^\cE(s'|s,a))_{s,a}$. Then clearly all the conditions $\|\phi\|_2 \le 1, \|\sum_{s'} \psi(s')\|_2 \le C_\psi$ and most importantly $P_h(\cdot|s,a) = \langle \phi(s,a) , \psi_h^\cE(s') \rangle$, are satisfied. However we note that for any $\cE,h$ we have $\sum_{s'} \psi_h^\cE(s') = (\sum_s' P_h^\cE(s'|s,a))_{s,a} = (1)_{s,a}$ which is the all one vector in $\mbbR^{SA}$. In that case, $\sum_{s'} \psi_h^\cE(s') - \sum_{s'} \psi_h^{\cE'}(s') $ in \cref{eq:the_wrong_estimate_tor} is the zero vector, with zero norm. 

Therefore, were the estimation in \cite[App. B.4]{hao2022regret} correct, for all tabular Bayesian RLs with deterministic reward, the difference of all value functions of the form $V^{\cE_1}_{1,\pi^*_{\cE_1}}-V^{\cE_2}_{1,\pi^*_{\cE_2}}$ would be bounded above by zero, meaning we have estimated $K_{\op{surr}}(\err) = 1$ for all $\err$, which since $\log(1)=0$,  implies a constant regret bound as well. This counterexample further demonstrates the mistake above.

Overall, this makes the proof for \cite[Theorem 4.10]{hao2022regret} incorrect and invalidates their claim of a regret bound $\widetilde{O}(d_f^{\op{hom}}H^{3/2}\sqrt{T})$. 

\begin{rmk}
    Another important gap in the proof of \cite[Theorem 4.10]{hao2022regret} can be found in \cite[App. B.5]{hao2022regret}, where the surrogate regret is claimed to be bounded by a conditional mutual information by $\pi^*$ instead of $\pi^\ell_\ts$. This is explained in further details in \cref{appsssec:mutual_information_rewrite_wrong}.
\end{rmk}

\section{Proof of Corollary \ref{cor:finite_mixtures_RL}}\label{appsec:finite_mixtures_RL}
We restate the corollary.
\begin{cor*}
    Given a finite mixtures Bayesian RL problem, for large enough $T$,
\begin{align}
\BR_L(\pi_\ts)\leq \widetilde{O}(\lambda\sqrt{d_{l_1}^mT})\,.
\end{align}
Assuming the restricted finite mixtures model, for large enough $T$,
\begin{align}\label{eq:app_finite_mixtures_indep_s_a}
    \BR_L(\pi_\ts)\leq \widetilde{O}\left(\lambda\sqrt{MT}\right)\,.
\end{align}
which, given a uniform dimension $m=m_h^P=m_h^R$, yields $\widetilde{O}(\lambda\sqrt{HmT})$.
\end{cor*}
Given \cref{thm:gen_TS_bd} and \cref{fact:surr_l_1_estimation_bayesian}, we need to estimate the $d_{l_1}$ of $\Theta$ by that of $\{\bm{a}_{h}^{P,\cE}(s,a)\}_{\cE \in \Theta}$ and $\{\bm{a}_{h}^{R,\cE}(s,a)\}_{\cE \in \Theta}$. Indeed writing the $l_1-$distance of two transition functions:
\begin{align}
\begin{split}
    &\sup_{s,a} \|P_h^\cE(\cdot|s,a)- P_h^{\cE'}(\cdot|s,a)\|_1 = \sup_{s,a} \|\sum_{i=1}^{m_h}({a}_{h,i}^{P,\cE}(s,a)-a_{h,i}^{P,\cE'}(s,a)) Z_{h,i}^P(\cdot|s,a)\|_1 \le \\
    &\sup_{s,a}\sum_{i=1}^{m_h} \|(a_{h,i}^{P,\cE}(s,a)-a_{h,i}^{P,\cE'}(s,a)) Z_{h,i}^P(\cdot|s,a)\|_1 =\sup_{s,a}\sum_{i=1}^{m_h} |a_{h,i}^{P,\cE}(s,a)-a_{h,i}^{P,\cE'}(s,a)| = \\
    &\sup_{s,a} \|\bm{a}_{h}^{P,\cE}(s,a)-\bm{a}_{h}^{P,\cE'}(s,a)\|_1
\end{split}
\end{align}
where we used the triangle inequality and the fact that the density functions are positive and their integral equals one. 

For the second statement, we are faced with the problem of finding an $l_1-$covering number for a collection of vectors on the $m_h-$dimensional simplex. It is a standard fact that the covering of the latter is of order $O\left((\frac{1}{\err})^{m_h}\right)$, implying \cref{eq:app_finite_mixtures_indep_s_a}.

\section{Useful Lemmas}\label{appsec:useful_lemmas}

\begin{lem}\label{lem:D_3_Tor_generalized_full_bayesian}
 For any two environments $\cE, \cE'$ with potentially different transition and reward functions, and any policy $\pi$, we have
\begin{align*}
V_{1, \pi}^{\cE}(s_1) &- V_{1, \pi}^{\cE'}(s_1) \\
&= \sum_{h=1}^H \mathbb E_{\pi}^{ \cE'}\left[\mathbb E_{s'\sim P_h^{ \cE}(\cdot|s_h,a_h)}[V_{h+1,\pi}^{ \cE}(s')] - \mathbb E_{s'\sim P_h^{ \cE'}(\cdot|s_h,a_h)}[V_{h+1,\pi}^{ \cE}(s')]\right] \\
&\quad+ \sum_{h=1}^H \E_\pi^{\cE'}[r_h^\cE(s_h,a_h)-r_h^{\cE'}(s_h,a_h)] \\
&= \sum_{h=1}^H \mathbb E_{\pi}^{ \cE'}\left[\mathbb E_{(s',r') \sim (P_h^{\cE} \otimes r_h^{\cE})(\cdot|s_h,a_h)}[r' + V_{h+1,\pi}^{ \cE}(s')] \right.\\
&\qquad\qquad\qquad\left. - \mathbb E_{(s', r') \sim (P_h^{\cE'} \otimes r_h^{\cE'})(\cdot|s_h,a_h)}[r' + V_{h+1,\pi}^{ \cE}(s')]\right],
\end{align*}
where $V_{H+1,\pi^*}^{ \cE}(\cdot):=0$ and the expectation $\mathbb E_{\pi}^{\cE'}$ is with respect to $s_h,a_h$.
\end{lem}
Note that when rewards are deterministic, we have
\begin{align*}
V_{1, \pi}^{\cE}(s_1) - V_{1, \pi}^{\cE'}(s_1)
&= \sum_{h=1}^H \mathbb E_{\pi}^{ \cE'}\left[\mathbb E_{s'\sim P_h^{ \cE}(\cdot|s_h,a_h)}[V_{h+1,\pi}^{ \cE}(s')] - \mathbb E_{s'\sim P_h^{ \cE'}(\cdot|s_h,a_h)}[V_{h+1,\pi}^{ \cE}(s')]\right],
\end{align*}
which is the statements of \cite[Lemma D.3]{hao2022regret}.
\begin{proof}
We have
\begin{align}
V_{1, \pi}^{\cE}(s_1) - V_{1, \pi}^{\cE'}(s_1) 
= (V_{1, \pi}^{\cE}(s_1)-V_{1, \pi}^{\cE'_{r^\cE}}(s_1)) 
    + (V_{1, \pi}^{\cE'_{r^\cE}}(s_1)-V_{1, \pi}^{\cE'}(s_1))
\end{align}
where $\cE'_{r^\cE}$ is an environment with the transition functions of $\cE'$ but reward functions of $\cE$.
The first term above may be rewritten using~\cite[Lemma D.3]{hao2022regret}. The second term may be rewritten using the direct definition of value function as $V_{1,\pi}^\cG = \E_\pi^\cG[\sum_{h=1}^H r^\cG(s_h,a_h)]$, which completes the proof.
\end{proof}

Pinsker's lemma is at the center of relating the two concepts of regret and mutual information.
We cite the following variant of the Pinsker's inequality from Fact~9 in~\cite{russo2014learning}.
\begin{lem}\label{lem:pinsker}
For any distribution $P$ and $Q$ such that $P$ is absolutely continuous with respect to $Q$, any random variable $X:\Omega \to \cX$ and any $g:\cX\to \mathbb R$ such that $\sup g-\inf g\leq 1$, we have 
\begin{align}
\mathbb E_P[g(x)]-\mathbb E_Q[g(x)] \leq \sqrt{\frac{1}{2}D_{\KL}(P||Q)}\,.
\end{align}
\end{lem}

\section{Mutual information of surrogate environment and history} \label{appssec:mutual_information_rewrite}
Recall that by performing the information ratio trick, Cauchy-Schwarz and Pinsker's inequality, we obtained the following term in our bound of the squared regret: 
\begin{align*}
\frac{1}{2}\sum_{h=1}^H \mathbb E_{\ell}\left[\mathbb E_{\pi^\ell_\ts}^{\bar \cE_\ell}\left[D_{\KL}\left(
    (P_h^{\tilde\cE^*_\ell} \otimes r_h^{\tilde\cE^*_\ell}) (\cdot|s_{h-1}^\ell,a_{h-1}^\ell)
    ||
    (P_h^{\bar \cE_\ell} \otimes r_h^{\bar \cE_\ell})(\cdot|s_{h-1}^\ell,a_{h-1}^\ell)
    \right)\right]\right].
\end{align*}
Now we would like to show that the above is $\frac{1}{2}\mbbI^{\pi^\ell_\ts}_\ell(\tilde\cE^*_\ell ; \cH_{\ell,H})$. To be more careful in our arguments, we need to be reminded of what the random variable $\cH_{\ell,H}$ is. We must view it as $\cH_{\ell,H} = \cH_{\ell,H}(\cE,\pi^\ell_\ts)$ or $\cH_{\ell,H}(\cE,\pi^*_{\cE_\ts})$ where $\cE,\cE_\ts$ are two independent samples of $\mbbP_\ell(\cdot)$, and $\cE$ represents the same $\cE$ in the regret above in $\pi^*_\cE$, i.e. the true environment. Also note that $\cE,\tilde\cE^*_\ell$ are dependent, as we set $\tilde\cE^*_\ell, \cE$ to have the same $\zeta$ value. 
\begin{lem}\label{lem:mutual_information_rewrite}
With $\tilde\cE^*_\ell$ defined according to the proof in \cref{appsec:lemma_surrogate_learning}, in a Bayesian RL, we have
\begin{align*}
\mbbI_\ell^{\pi^\ell_\ts}(\tilde \cE^*_\ell;\cH_{\ell,H}) 
= \sum_{h=1}^H \mathbb E_{\ell}\left[\mathbb E_{\pi^\ell_\ts}^{\bar \cE_\ell}\left[D_{\KL}\left(
    (P_h^{\tilde\cE^*_\ell} \otimes r_h^{\tilde\cE^*_\ell}) (\cdot|s_{h-1}^\ell,a_{h-1}^\ell)
    ||
    (P_h^{\bar \cE_\ell} \otimes r_h^{\bar \cE_\ell})(\cdot|s_{h-1}^\ell,a_{h-1}^\ell)
    \right)\right]\right]
\end{align*}
As a special case, when rewards are deterministic, we have
\begin{align*}
\mbbI_\ell^{\pi^\ell_\ts}(\tilde \cE^*_\ell;\cH_{\ell,H})
= \sum_{h=1}^H \E_\ell\left[\E_{\pi^\ell_\ts}^{\bar\cE_\ell}\left[\DKL(P_h^{ \tilde \cE_\ell^*}(\cdot|s_h,a_h)||P_h^{ \bar\cE_\ell}(\cdot|s_h,a_h)\right]\right].
\end{align*}
\end{lem}
\begin{proof}
Using the chain rule of mutual information,
\begin{equation}\label{def:mutual_information_decom}
\begin{split}
\mbbI^{\pi^\ell_\ts}_\ell(\tilde\cE^*_\ell ; \cH_{\ell,H}) 
&= \sum_{h=1}^{H} \mbbI_\ell^{\pi^\ell_\ts}\left(\tilde\cE^*_\ell; (s_{h}^\ell, a_h^\ell, r_{h}^\ell)\big|\cH_{\ell, h-1}\right) \\
&= \sum_{h=1}^{H}\mathbb I_\ell^{\pi^\ell_\ts}\left(\tilde\cE^*_\ell; s_{h}^\ell\big|\cH_{\ell, h-1}\right)+\sum_{h=1}^{H} \mathbb I_\ell^{\pi^\ell_\ts}\left(\tilde\cE^*_\ell; a_{h}^\ell\big|s_{h}^\ell, \cH_{\ell, h-1}\right) \\
&\quad+ \sum_{h=1}^{H}  \mathbb I_\ell^{\pi^\ell_\ts}\left(\tilde\cE^*_\ell; r_{h}^\ell\big|s_{h}^\ell, a_h^\ell,  \cH_{\ell, h-1}\right)\,.
\end{split}
\end{equation}
Let us note what is meant by $\mbbI^{\pi^\ell_\ts}_\ell(\cdot)$ is $\mbbI_\ell(\cdot | \pi^\ell_\ts)$.  In \cite{hao2022regret}, the policy $\pi$ used is fixed/independent from the random variables involved in the mutual information (given $\cD_\ell$). Here, the same holds as $\pi^\ell_\ts$ and $\tilde\cE^*_\ell$ are independent. 
\begin{itemize}
    \item For the first term in Eq.~\eqref{def:mutual_information_decom}, by using $\mbbI(X;Y) = \int \DKL(P(Y|x)||P(Y)) \d\mbbP(x)$, and the definition of conditional mutual information, we have  
\begin{equation}\label{eqn:bound_mutual_information_s}
\small
\begin{split}
\mathbb I_\ell &\left(\tilde\cE^*_\ell; s_{h}^\ell\big|\cH_{\ell, h-1},\pi^\ell_\ts\right) \\
&= \int \int D_{\KL}\left(
    \mathbb P_{\ell}\left(s_h^\ell=\cdot|\cH_{\ell, h-1},\pi^*_{\cE_\ts}, \tilde\cE^*_\ell\right)||\mathbb P_{\ell}\left(s_h^\ell = \cdot|\cH_{\ell, h-1},\pi^*_{\cE_\ts}\right)
\right) \\
&\qquad\qquad\d \mathbb P_\ell(\tilde\cE^*_\ell|\cH_{\ell, h-1},\pi^*_{\cE_\ts})\d \mbbP_\ell(\cH_{\ell,h-1},\pi^*_{\cE_\ts}) \\
&= \int \int D_{\KL}\left(
    P_h^{\tilde\cE^*_\ell}\left(\cdot|s_{h-1}^\ell, a_{h-1}^\ell\right)||\mathbb P_{\ell}\left(s_h^\ell = \cdot|\cH_{\ell, h-1},\pi^*_{\cE_\ts}\right)
\right) \\
&\qquad\qquad \d \mathbb P_\ell(\tilde\cE^*_\ell|\cH_{\ell, h-1},\pi^*_{\cE_\ts}) \d \mbbP_\ell(\cH_{\ell,h-1},\pi^*_{\cE_\ts})\,.
\end{split}
\end{equation}
Where we substituted $\mathbb P_{\ell}\left(s_h^\ell=\cdot|\cH_{\ell, h-1},\pi^*_{\cE_\ts}, \tilde\cE^*_\ell\right) = P_h^{\tilde\cE^*_\ell}\left(\cdot|s_{h-1}^\ell, a_{h-1}^\ell\right)$. Let us see why this is the case. Let us analyze the meaning of the  conditional on $(\cH_{\ell, h-1},\pi^*_{\cE_\ts}, \tilde\cE^*_\ell)$. Recall that $\cH_{\ell, h-1}=\cH_{\ell, h-1}(\cE,\pi^*_{\cE_\ts})$. Since $\tilde\cE^*_\ell$ is given, the random variable $\cE$ can only go over the partition $\Theta_{\zeta(\tilde\cE^*_\ell)}^\err$. Of course, we can also drop all conditionals on previous state transitions, except for the last one $s_{h-1}^\ell(\cE), a_{h-1}^\ell(\cE)$. Note that the policy $\pi^*_{\cE_\ts}$ is also irrelevant in this conditional, since the next state only depends on probability transitions and not on policy, hence why also we are not using the full notation $s_{h-1}^\ell(\cE,\pi^*_{\cE_\ts}), a_{h-1}^\ell(\cE,\pi^*_{\cE_\ts})$. This implies
\begin{align}
\mbbP_{\ell}\left(s_h^\ell(\cE) = \cdot|\cH_{\ell, h-1},\pi^*_{\cE_\ts}, \tilde\cE^*_\ell\right) 
&= \mbbP_{\ell}\left(s_h^\ell(\cE)=\cdot|s_{h-1}^\ell(\cE), a_{h-1}^\ell(\cE), \tilde\cE^*_\ell\right) \\
&= \int_\cE P_h^{\cE}\left(s_h^\ell=\cdot|s_{h-1}^\ell, a_{h-1}^\ell\right) \d\mbbP_\ell(\cE|\zeta(\cE) \\
&= \zeta(\tilde \cE^*_\ell)) \\
&= P_h^{\E_\ell[\cE|\zeta(\cE) = \zeta(\tilde \cE^*_\ell)]}\left(s_h^\ell=\cdot|s_{h-1}^\ell, a_{h-1}^\ell\right)  
\end{align}
However, recall that we defined $\tilde\cE^*_\ell$ to be the posterior mean of $\cE$ over $\Theta_k^\err$, i.e. $\E_\ell[\cE|\zeta(\cE) =\zeta(\tilde\cE^*_\ell)] = \tilde\cE^*_\ell$. Hence, the average above yields $P_h^{\tilde\cE^*_\ell}\left(\cdot|s_{h-1}^\ell, a_{h-1}^\ell\right)$, as desired. Next, for the second term in the KL-divergence,
\begin{equation}\label{eqn:s_h}
\begin{split}
    \mathbb P_{\ell}\left(s_h^\ell(\cE) = \cdot|\cH_{\ell, h-1},\pi^*_{\cE_\ts}\right) &=\int \mathbb P_{\ell}\left(s_h^\ell(\cE)=\cdot|\cH_{\ell, h-1}, \pi^*_{\cE_\ts}, \cE\right) \d \mathbb P_\ell(\cE|\cH_{\ell, h-1}, \pi^*_{\cE_\ts})\\
    &=\int P_h^{\cE}(s_h^\ell =\cdot|s_{h-1}^\ell, a_{h-1}^\ell) \d \mathbb P_\ell(\cE|\cH_{\ell, h-1}, \pi^*_{\cE_\ts})\\
    &=\int P_h^{\cE}(s_h^\ell = \cdot|s_{h-1}^\ell, a_{h-1}^\ell) \d \mathbb P_\ell(\cE)\\
     &= P_h^{\bar \cE_\ell}\left(s_h^\ell = \cdot|s_{h-1}^\ell,a_{h-1}^\ell\right)\,.
\end{split}
\end{equation}
In the above equations, $\cH_{\ell,h-1},\pi^*_{\cE_\ts}$ are given in the conditional, and the true environment $\cE$ is being integrated. The second equality was explained in the previous case. Let us explain why $ \d \mathbb P_\ell(\cE|\cH_{\ell, h-1}, \pi^*_{\cE_\ts}) = \d \mathbb P_\ell(\cE)$ in the third equality. Due to the independence of priors over different layers, the conditional on $\cH_{\ell, h-1}$ impacts transition functions of prior layers (i.e. $P_1^\cE,\ldots,P_{h-1}^\cE$), while the transition function in question is the one at layer $h$. Therefore, this conditional can be dropped, as well as $\pi^*_{\cE_\ts}$ since $\cE_\ts,\cE$ are two independent samples of $\mbbP_\ell$. Finally, the last equation is by the definition of probability kernel $P_h^{\bar \cE_\ell}$. Eqs.~\eqref{eqn:bound_mutual_information_s} and \eqref{eqn:s_h} imply $\mathbb I_\ell^{\pi^\ell_\ts}\left(\tilde\cE^*_\ell; s_{h}^\ell\big|\cH_{\ell, h-1}, \pi^*_{\cE_\ts}\right) =$
\begin{align}
 \int \int D_{\KL}\left(P_h^{\tilde\cE^*_\ell}(\cdot|s_{h-1}^\ell,a_{h-1}^\ell)||P_h^{\bar \cE_\ell}(\cdot|s_{h-1}^\ell,a_{h-1}^\ell)\right)\d \mathbb P_\ell(\tilde\cE^*_\ell) \d \mbbP_\ell(\cH_{\ell,h-1},\pi^*_{\cE_\ts})\,.
\end{align}
where we note we also dropped the conditionals on $\cH_{\ell,h-1},\pi^*_{\cE_\ts}$ in $\d \mathbb P_\ell(\tilde\cE^*_\ell)$, by the similar argument in the previous case for $\d \mathbb P_\ell(\cE|\cH_{\ell, h-1}, \pi^*_{\cE_\ts})$ as the integrand is transitions at the $h-$th step. We continue by focusing on the outer integral with respect to 
\begin{align}
\d \mbbP_\ell(\cH_{\ell,h-1},\pi^*_{\cE_\ts}) 
= P(\cH_{\ell,h-1}(\cE,\pi^*_{\cE_\ts})|\cE,\pi^*_{\cE_\ts})
    \d\mu_{\cS\times\cA}^{\otimes (h-1)}\d\mbbP_\ell(\cE) \d\mbbP_\ell(\pi^*_{\cE_\ts})
\end{align}
and note that since only transitions at the $(h-1)$-th step are inside the inner integral, one can marginalize prior ($h-2,\ldots,1$) state-action-reward tuples, yielding
\begin{equation*}
\begin{split}
&= \int_{s_{h-1}^\ell,a_{h-1}^\ell,\cE,\pi^*_{\cE_\ts}} 
    P(s_{h-1}^\ell,a_{h-1}^\ell|\cE, \pi^*_{\cE_\ts}) \\
&\qquad\qquad\left(\int_{\tilde\cE^*_\ell} D_{\KL}\left(P_h^{\tilde\cE^*_\ell}(\cdot|s_{h-1}^\ell,a_{h-1}^\ell)||P_h^{\bar \cE_\ell}(\cdot|s_{h-1}^\ell,a_{h-1}^\ell)\right)\d \mathbb P_\ell(\tilde\cE^*_\ell)\right) \\
&\qquad\qquad\qquad\qquad \d\mu_{\cS\times\cA}\d\mbbP_\ell(\cE) \d\mbbP_\ell(\pi^*_{\cE_\ts}) \\
&= \int_{s_{h-1}^\ell,a_{h-1}^\ell, \pi^*_{\cE_\ts},\tilde \cE^*_\ell}
\left( \int_\cE P(s_{h-1}^\ell,a_{h-1}^\ell|\cE, \pi^*_{\cE_\ts})\d\mbbP_\ell(\cE)\right) \\
&\qquad\qquad D_{\KL}\left(P_h^{\tilde\cE^*_\ell}(\cdot|s_{h-1}^\ell,a_{h-1}^\ell)||P_h^{\bar \cE_\ell}(\cdot|s_{h-1}^\ell,a_{h-1}^\ell)\right)
\d \mathbb P_\ell(\tilde\cE^*_\ell,\pi^*_{\cE_\ts}) \d\mu_{\cS\times\cA}
\end{split}
\end{equation*}
where we simply rearranged the measures and integrals, and note the independence $ \mathbb P_\ell(\tilde\cE^*_\ell) \mbbP_\ell(\pi^*_{\cE_\ts})= \mathbb P_\ell(\tilde\cE^*_\ell,\pi^*_{\cE_\ts})$. For the outer integral, notice that $P(s_{h-1}^\ell,a_{h-1}^\ell|\cE, \pi^*_{\cE_\ts}) = d_{h,\pi^*_{\cE_\ts}}^{\cE}(s_{h-1}^\ell, a_{h-1}^\ell)$ by definition. So using the linearity of expectation and independence of priors over different layers
\begin{align}
    &\int_\cE d_{h,\pi^*_{\cE_\ts}}^{\cE}(s_{h-1}^\ell,a_{h-1}^\ell)\d\mbbP_\ell(\cE) = d_{h,\pi^*_{\cE_\ts}}^{ \bar\cE_\ell}(s_{h-1}^\ell, a_{h-1}^\ell)\,.
\end{align}
Putting it all together, and going back to the notation $\pi^*_{\cE_\ts} \to \pi^\ell_\ts$ :
  \begin{equation*}
\begin{split}      
& \mathbb I_\ell\left(\tilde\cE^*_\ell; s_{h}^\ell\big|\cH_{\ell, h-1},\pi^\ell_\ts\right) \\ 
&= \int_{s_{h-1}^\ell,a_{h-1}^\ell, \pi^\ell_\ts,\tilde \cE^*_\ell}d_{h,\pi^\ell_\ts}^{ \bar\cE_\ell}(s_{h-1}^\ell, a_{h-1}^\ell) \\
&\qquad\qquad D_{\KL}\left(P_h^{\tilde\cE^*_\ell}(\cdot|s_{h-1}^\ell,a_{h-1}^\ell)||P_h^{\bar \cE_\ell}(\cdot|s_{h-1}^\ell,a_{h-1}^\ell)\right)\d \mathbb P_\ell(\tilde\cE^*_\ell,\pi^\ell_\ts) \d \mu_{\cS\times \cA}\\
&= \int_{\pi^\ell_\ts,\tilde \cE^*_\ell}\mathbb E_{\pi^\ell_\ts}^{\bar \cE_\ell}\left[D_{\KL}\left(P_h^{\tilde\cE^*_\ell}(\cdot|s_{h-1}^\ell,a_{h-1}^\ell)||P_h^{\bar \cE_\ell}(\cdot|s_{h-1}^\ell,a_{h-1}^\ell)\right)\right]\d \mathbb P_\ell(\tilde\cE^*_\ell,\pi^\ell_\ts)\\
&= \mathbb E_{\ell}\left[\mathbb E_{\pi^\ell_\ts}^{\bar \cE_\ell}\left[D_{\KL}\left(P_h^{\tilde\cE^*_\ell}(\cdot|s_{h-1}^\ell,a_{h-1}^\ell)||P_h^{\bar \cE_\ell}(\cdot|s_{h-1}^\ell,a_{h-1}^\ell)\right)\right]\right]\,,
\end{split}
\end{equation*}
where $\mathbb E_{\pi^\ell_\ts}^{\bar \cE_\ell}$ is taken with respect to $s_{h-1}^\ell,a_{h-1}^\ell$ and $\mathbb E_\ell$ is taken with respect to $\pi^\ell_\ts,\tilde\cE^*_\ell$.
   \item For the second term in \cref{def:mutual_information_decom}, 
   \begin{equation*}
   \begin{split}
       &\mathbb I_\ell\left(\tilde\cE^*_\ell; a_{h}^\ell\big|s_{h}^\ell, \cH_{\ell, h-1},\pi^\ell_\ts\right)\\
       &= \int \int D_{\KL}\left(\mathbb P_{\ell}\left(a_h^\ell=\cdot|s_h^\ell,\cH_{\ell, h-1},\pi^\ell_\ts, \tilde\cE^*_\ell\right)||\mathbb P_{\ell}\left(a_h^\ell=\cdot|s_h^\ell, \cH_{\ell, h-1},\pi^\ell_\ts\right)\right)\,.
       \end{split}
   \end{equation*}
  where the integrals are with respect to $\d \mathbb P_\ell(\tilde\cE^*_\ell|s_h^\ell,\cH_{\ell, h-1},\pi^\ell_\ts)\d \mbbP_\ell(s_h^\ell,\cH_{\ell, h-1},\pi^\ell_\ts)$. When $s_h^\ell,\pi^\ell_\ts$ are given, both sides of the KL term are equal to $\pi^\ell_\ts(\cdot|s_h^\ell)$ and thus the above is zero.
  \item For the third term, we use an argument similar to the first term to see that 
  \begin{align}
      \mathbb I_\ell\left(\tilde\cE^*_\ell; r_{h}^\ell\big|s_{h}^\ell, a_h^\ell,  \cH_{\ell, h-1}, \pi^\ell_\ts\right) =  \mathbb E_{\ell}\left[\mathbb E_{\pi^\ell_\ts}^{\bar \cE_\ell}\left[D_{\KL}\left(r_h^{\tilde\cE^*_\ell}(\cdot|s_h^\ell,a_h^\ell)||r_h^{\bar \cE_\ell}(\cdot|s_h^\ell,a_h^\ell)\right)\right]\right]\,.
  \end{align}
  
\end{itemize}    
Put together, it follows that
\begin{equation*}
\begin{split}
&\mathbb I_\ell^{\pi^\ell_\ts}\left(\tilde\cE^*_\ell; \cH_{\ell, H}\right) \\
&\quad= \sum_{h=1}^H \mathbb E_{\ell}\left[\mathbb E_{\pi^\ell_\ts}^{\bar \cE_\ell}\left[D_{\KL}\left(P_h^{\tilde\cE^*_\ell}(\cdot|s_{h-1}^\ell,a_{h-1}^\ell)||P_h^{\bar \cE_\ell}(\cdot|s_{h-1}^\ell,a_{h-1}^\ell)\right)\right]\right] \\
&\quad\quad+ \mathbb E_{\ell}\left[\mathbb E_{\pi^\ell_\ts}^{\bar \cE_\ell}\left[D_{\KL}\left(r_h^{\tilde\cE^*_\ell}(\cdot|s_h^\ell,a_h^\ell)||r_h^{\bar \cE_\ell}(\cdot|s_h^\ell,a_h^\ell)\right)\right]\right] \\
&\quad= \sum_{h=1}^H \mathbb E_{\ell}\left[\mathbb E_{\pi^\ell_\ts}^{\bar \cE_\ell}\left[D_{\KL}\left(
    (P_h^{\tilde\cE^*_\ell} \otimes r_h^{\tilde\cE^*_\ell}) (\cdot|s_{h-1}^\ell,a_{h-1}^\ell)
    ||
    (P_h^{\bar \cE_\ell} \otimes r_h^{\bar \cE_\ell})(\cdot|s_{h-1}^\ell,a_{h-1}^\ell)
    \right)\right]\right]\,.
\qedhere
\end{split}
\end{equation*}
\end{proof}

\subsection{On the rewrite of mutual information in \cite[App. B.5]{hao2022regret}}\label{appsssec:mutual_information_rewrite_wrong}
We start by citing the relevant equations involved in \cite[App. B.5]{hao2022regret}. For $\Sigma_h = \mathbb E_{\ell}\left[\mathbb E_{\pi^*}^{\bar \cE_\ell^*}\left[\phi(s_h^\ell, a_h^\ell)\right]\mathbb E_{\pi^*}^{\bar \cE_\ell^*}\left[\phi(s_h^\ell, a_h^\ell)^{\top}\right]\right]$, the authors claim
\begin{align}
\sum_{h=1}^H\mathbb E_\ell &\left[\left\|\Sigma_h^{1/2}\sum_{s'}(\psi_h^{\tilde \cE_\ell^*}(s')- \psi_h^{\bar \cE_\ell^*}(s'))V_{h+1,\pi^*}^{ \tilde \cE_\ell^*}(s')\right\|_2^2\right] \\
&= \mathbb E_{\ell}\left[\sum_{h=1}^H\mathbb E_{\pi^*}^{\bar \cE_\ell}\left[\left(P_h^{ \tilde \cE_\ell^*}(\cdot|s_h^\ell,a_h^\ell)^{\top}V_{h+1,\pi^*}^{\tilde \cE_\ell^*}(\cdot)-P_h^{\bar \cE_\ell^*}(\cdot|s_h^\ell,a_h^\ell)^{\top}V_{h+1,\pi^*}^{\tilde \cE_\ell^*}(\cdot)\right)^2\right]\right] \\\label{eq:wrong_mutual_information_eq}
&\le \frac{1}{2}\sum_{h=1}^H \mathbb E_{\ell}\left[\mathbb E_{\pi^*_\cE}^{\bar \cE_\ell}\left[D_{\KL}\left(P_h^{\tilde\cE^*_\ell}(\cdot|s_{h-1}^\ell,a_{h-1}^\ell)||P_h^{\bar \cE_\ell}(\cdot|s_{h-1}^\ell,a_{h-1}^\ell)\right)\right]\right] \\
&= \frac{1}{2}\mathbb I_\ell^{\pi^*}\left(\tilde \cE_\ell^*; \cH_{\ell, H}\right)\,.
\end{align}
For the last part in \cref{eq:wrong_mutual_information_eq}, the authors do not provide a proof, and cite their own \cite[Lemma A.1]{hao2022regret} as support. However, that lemma is for $\mbbI^\pi_\ell(\cE;\cH_{\ell,H})$, where $\pi$ is the \textbf{algorithm} and not the optimal policy of the true environment. 

We need a rewrite of $\mathbb I_\ell^{\pi^*_\cE}\left(\tilde \cE_\ell^*; \cH_{\ell, H}\right)$. We emphasize that in the former mutual information expression, the policy involved is the algorithm $\pi$, which clearly is not dependent on the true environment $\cE$, unlike $\pi^*_\cE$. Furthermore, the environment involved is also independent from the policy, but that is not the case here since $\tilde\cE^*_\ell,\pi^*_\cE$ are dependent through $\zeta$. As we shall see, it is crucial for the policy in the mutual information expression to be independent from the true environment, in order for the argument in \cite[Lemma A.1]{hao2022regret}.

Since our own lemma above for $\mathbb I_\ell^{\pi^\ell_\ts}\left(\tilde \cE_\ell^*; \cH_{\ell, H}\right)$ naturally extends \cite[Lemma A.1]{hao2022regret} and takes the first step for the substitution of $\cE$ by $\tilde\cE^*_\ell$, we can analyze what happens in our own equations in \cref{lem:mutual_information_rewrite}, assuming we were to take $\pi^*_\cE$, the optimal policy of the true environment $\cE$, instead of $\pi^\ell_\ts$. We can apply the mutual information chain rule as before, and focus on
\begin{align*}
\mathbb I_\ell\left(\tilde\cE^*_\ell; s_{h}^\ell\big|\cH_{\ell, h-1},\pi^*_\cE\right) 
= \int \int D_{\KL}\left(\mathbb P_{\ell}(s_h^\ell=\cdot|\cH_{\ell, h-1},\pi^*_\cE, \tilde\cE^*_\ell)||\mathbb P_{\ell}\left(s_h^\ell = \cdot|\cH_{\ell, h-1},\pi^*_\cE\right)\right) \\
\d \mathbb P_\ell(\tilde\cE^*_\ell|\cH_{\ell, h-1},\pi^*_\cE)
\d \mbbP_\ell(\cH_{\ell,h-1},\pi^*_\cE)
\end{align*}
Recall that the history is of the form $\cH_{\ell,H} = \cH_{\ell,H}(\cE,\pi^*_{\cE_\ts})$. The first thing to prove above should be $\mathbb P_{\ell}\left(s_h^\ell=\cdot|\cH_{\ell, h-1},\pi^*_{\cE}, \tilde\cE^*_\ell\right) = P_h^{\tilde\cE^*_\ell}\left(\cdot|s_{h-1}^\ell, a_{h-1}^\ell\right)$ . However, note that since $\pi^*_\cE$ is given, this means $\cE$ is given, at least in (realistic) scenarios where there is uniqueness of optimal policies, and as a result, the true environment $\cE$ in $\cH_{\ell,H}(\cE,\pi^*_{\cE_\ts})$ is determined uniquely. This implies that in fact $\mathbb P_{\ell}\left(s_h^\ell=\cdot|\cH_{\ell, h-1},\pi^*_{\cE}, \tilde\cE^*_\ell\right) = P_{h}^\cE\left(s_h^\ell=\cdot|s_{h-1}^\ell,a_{h-1}^\ell\right) $. In general it would be the average $P_{h}^{\E[\cE'|\pi^*_{\cE'}=\pi^*_\cE, \zeta(\cE')=\zeta(\tilde\cE^*_\ell)]}\left(s_h^\ell=\cdot|s_{h-1}^\ell,a_{h-1}^\ell\right)$. 

Either way, in the very first step, we have shown that the dependence of the policy with the true environment can alter significantly the rewrite of the mutual information by the argument in \cite[Lemma A.1]{hao2022regret}. Clearly, this does not lead to the desired rewrite in \cref{eq:wrong_mutual_information_eq} and makes this claimed bound of the Bayesian regret by that mutual information (at the very least) unproven. 

A more direct way to note the gap in the argument is the following. Recall that the denominator in the (surrogate) information ratio is supposed to represent the information gain by the algorithm on the true (or surrogate) environment. The surrogate mutual information ratio that one should bound is:
\begin{align}
    \frac{(\E_\ell\left[V_{1,\pi^*_\cE}^{\tilde \cE_\ell^*}(s_1^\ell)-V_{1,\pi}^{\tilde \cE_\ell^*}(s_1^\ell)\right])^2}{\mathbb I_\ell^{\pi}\left(\tilde \cE_\ell^*; \cH_{\ell, H}\right)},
\end{align}
where $\pi$ is the algorithm (and we select $\pi=\pi_\ts$). Clearly the algorithm can not know about the true environment $\cE$, which makes it questionable to try to bound the surrogate regret $\E_\ell\left[V_{1,\pi^*_\cE}^{\tilde \cE_\ell^*}(s_1^\ell)-V_{1,\pi}^{\tilde \cE_\ell^*}(s_1^\ell)\right]$ by a mutual information such as $\mathbb I_\ell^{\pi^*_\cE}\left(\tilde \cE_\ell^*; \cH_{\ell, H}\right):= \mathbb I_\ell\left(\tilde \cE_\ell^*; \cH_{\ell, H}|\pi^*_\cE \right)$, where there is assumed knowledge of the true environment in the conditional, as opposed to conditioning on the algorithm itself like in the ratio above. Therefore, the information ratio that the authors in \cite[App. B.5]{hao2022regret} are (implicitly) trying to bound is not the right one.

\section{Posterior consistency}\label{appsec:posterior_consistency}

\newcommand{\C}[1]{{\mathcal{#1}}} % math Calligraphy
\newcommand{\B}[1]{{\mathbb{#1}}} % math Black board
\newcommand{\BF}[1]{{\mathbf{#1}}} % math bold
\newcommand{\SC}[1]{{\mathscr{#1}}} 
\newcommand{\F}[1]{{\mathfrak{#1}}}

In this section we define the notion of posterior consistency and state Doob's consistency theorem.
We start by describing posterior consistency in a general setting.

Let $\SC{X}$ be a measure space and for every $n \in \B{N}$, let $X^{(n)}$ be an observation in the sample space $\SC{X}^{n}$ with distribution $P_\theta^{(n)}$ indexed by a parameter $\theta$ belonging to a separable metric space $\Omega$.
For instance $X^{(n)}$ might be a sample of size $n$ from a given distribution $P_\theta$ with $P_\theta^{(n)}$ the corresponding product measure.
Given a prior $\Pi$ on the Borel sets of $\Omega$, let $\Pi_n(\cdot \mid X^{(n)})$ be the posterior distribution given the observation $X^{(n)}$.
Moreover, we assume that there is a measure $P_\theta^{(\infty)}$ on $\SC{X}^{\infty}$ such that $P_\theta^{(n)}$ is equal to the the image $P_\theta^{(\infty)} \circ (X^{(n)})^{-1}$ of the probability measure $P_\theta^{(\infty)}$ when pushed forward onto $\SC{X}^{(n)}$.
We say an estimator $T := (T_n)_{n = 1}^\infty$, where $T_n : \SC{X}^{(n)} \to \Omega$ is a measureable function for all $n \geq 1$, is a strongly consistent estimator of $\theta$ if for every $\theta_0 \in \Omega$ and almost every $X^{(\infty)} = (X^{(n)})_{n = 1}^\infty$, we have
\[
\lim_{n \to \infty} T_n(X^{(n)}) = \theta_0.
\]

We can now describe the content of \cref{assumption:consistency}.
Let $\Omega := \Theta$ with the measure $\Pi := \rho$ as the prior, and let $\SC{X}$ be the space of all single-episode histories.
Also let $\theta_0 := \cE_0$ and $P_{\theta_0}^{(\ell)}(\cD_\ell) := \B{P}(\cD_\ell \mid \cE_0)$ be the probability of observing the history $\cD_\ell$ in the true environment $\cE_0$.
Existence of the measure $P_\theta^{(\infty)}$ on $\SC{X}^{\infty}$ as described above follows from the fact that for any $l' > l$, we have $P_{\theta_0}^{(\ell')}(\cD_\ell) = P_{\theta_0}^{(\ell)}(\cD_\ell) := \B{P}(\cD_\ell \mid \cE_0)$.
\cref{assumption:consistency} states that there exists a strongly consistent estimator of the true environment $T$ such that for almost every environment $\cE_0$ and almost every infinite history $\cD = (\cD_\ell)_{\ell = 1}^\infty$ sampled from the environment $\cE_0$, we have
\[
\lim_{n \to \infty} T_\ell(\cD_\ell) = \cE_0.
\]

The existence of consistent estimators is closely related to the notion of posterior consistency:
\begin{dfn}
The posterior distribution $\Pi_n(\cdot \mid X^{(n)})$ is said to be strongly consistent at $\theta_0 \in \Omega$ if for every neighbourhood $U$ of $\theta_0$ and $P_{\theta_0}^{(\infty)}$-almost every $X^{(\infty)}$, we have $\Pi_n(U^c \mid X^{(n)}) \to 0$ where $X^{(n)}$ is the projection of $X^{(\infty)}$ into the space $\SC{X}^n$.
\end{dfn}

Here we state a version of Doob's consistency theorem that we need for our application. (Theorem~6.9 in~\cite{ghosal2017fundamentals})

\begin{thm}[Doob's consistency theorem]\label{thm:doob}
If there is a strongly consistent estimator $T_n : \SC{X}^{(n)} \to \Omega$, then the posterior is strongly consistent at $\Pi$-almost every $\theta \in \Omega$.
In fact, $\int f(\theta') d\Pi_n(\theta' \mid X^{(n)}) \to f(\theta)$, almost surely $[P_\theta^{(\infty)}]$, for $\Pi$-almost every $\theta$ and every $\Pi$-integrable function $f$.
\end{thm}

Note that while this statement is not the exact statement of Theorem~6.9 in~\cite{ghosal2017fundamentals}, it is equivalent to it as discussed in the paragraph following the theorem.

\begin{cor}\label{cor:doob-bounded}
Given Assumption~\ref{assumption:consistency}, for any $\Pi$-integrable function $f : \Theta \to \B{R}$ and almost every $\cD_\infty$ sampled from true environment $\cE_0$, we have
\[
\lim_{\ell \to \infty} \B{E}_\ell [f(\cE)] = f(\cE_0).
\]
Similarly, if $f : \Theta \times \Theta \to \B{R}$ is bounded and $(\Pi \times \Pi)$-integrable, for almost every $\cD_\infty$ sampled from true environment $\cE_0$, we have
\[
\lim_{\ell \to \infty} \B{E}_\ell [f(\cE, \cE')] = f(\cE_0, \cE_0),
\]
where the expectation is taken over all values of $\cE$ and $\cE'$, sampled according to $\B{P}_\ell$.
\end{cor}
\begin{proof}
The first statement immediately follows from \cref{assumption:consistency} and \cref{thm:doob}.
To prove the second part, we use the first part to see that for any fixed value of $\cE' \in \Theta$ and almost every $\cD_\infty$, we have
\[
\lim_{\ell \to \infty} (\B{E}_\ell)_{\cE \sim \B{P}_\ell} [f(\cE, \cE')] = f(\cE_0, \cE').
\]
Now we use dominated convergence theorem to see that 
\[
\lim_{\ell \to \infty} (\B{E}_\ell)_{\cE, \cE' \sim \B{P}_\ell} [f(\cE, \cE')] 
= \lim_{\ell \to \infty} (\B{E}_\ell)_{\cE' \sim \B{P}_\ell} [f(\cE_0, \cE')] 
= f(\cE_0, \cE_0).
\qedhere
\]
\end{proof}

\end{document}